\pdfoutput=1
\documentclass[10pt,twocolumn,letterpaper]{article}
\PassOptionsToPackage{table,dvipsnames,x11names}{xcolor}
\usepackage[pagenumbers]{cvpr}
\usepackage[utf8]{inputenc}
\usepackage[T1]{fontenc}
\usepackage{amsfonts}
\usepackage{nicefrac}
\usepackage{microtype}
\usepackage{algorithm}
\usepackage{algpseudocode}
\algrenewcommand\algorithmicrequire{\textbf{Input:}}
\usepackage{needspace}
\usepackage{wrapfig}
\usepackage{multirow}
\usepackage{multicol}
\usepackage{pifont}
\usepackage{amsthm}
\usepackage{mathtools}
\usepackage{thmtools}
\usepackage{thm-restate}

\DeclareMathOperator*{\argmax}{arg\,max}

\newcommand{\mbx}{\mathbf{x}}

\newcommand{\mbz}{\mathbf{z}}

\newcommand{\mbp}{\mathbf{p}}
\newcommand{\mbv}{\mathbf{v}}
\newcommand{\mbxi}{\mathbf{\xi}}
\newcommand{\mbXi}{\mathbf{\Xi}}
\newcommand{\mt}{\mathcal{T}}

\newcommand{\mbc}{\mathbf{c}}
\newcommand{\me}{\boldsymbol{\epsilon}}

\definecolor{revblue}{RGB}{0,0,255}
\newif\ifshowrev\showrevfalse



\theoremstyle{plain}
\newtheorem{theorem}{Theorem}
\newtheorem{proposition}{Proposition}

\newtheorem{corollary}{Corollary}
\theoremstyle{definition}
\newtheorem{definition}{Definition}
\newtheorem{assumption}{Assumption}
\newtheorem{remark}{Remark}

\newcommand{\Att}{\texttt{At}}

\newcommand{\Kb}{\mathbf{K}}
\newcommand{\Qb}{\mathbf{Q}}
\newcommand{\Vb}{\mathbf{V}}

\newcommand{\Tc}{\mathcal{T}}
\newcommand{\Tca}{\mathcal{T}_\alpha}
\newcommand{\Tcal}{\mathcal{T}_\alpha^\lambda}
\newcommand{\Tcd}{\mathcal{T}_{\texttt{Dense}}}
\newcommand{\xib}{\boldsymbol{\xi}}

\newcommand{\norm}[1]{\left\lVert#1\right\rVert}

\usepackage{placeins}


\definecolor{revred}{RGB}{215,0,0}


\let\newred\rr
\providecommand{\rnow}[1]{#1}

\definecolor{cvprblue}{rgb}{0.21,0.49,0.74}
\usepackage[pagebackref,breaklinks,colorlinks,allcolors=cvprblue]{hyperref}
\def\confName{CVPR}
\def\confYear{2027}

\title{Guidance Ends When Reading Ends: Retrieval-Gated Attention Guidance for Diffusion Models through a Hopfield Lens}
\author{Kwanyoung Kim\\
Department of AI, GIST\\
{\tt\small k0.kim@gist.ac.kr}}

\begin{document}
\twocolumn[{%
\begin{@twocolumnfalse}
\maketitle
\vspace{-3.2em}
\begin{center}
\includegraphics[scale=1]{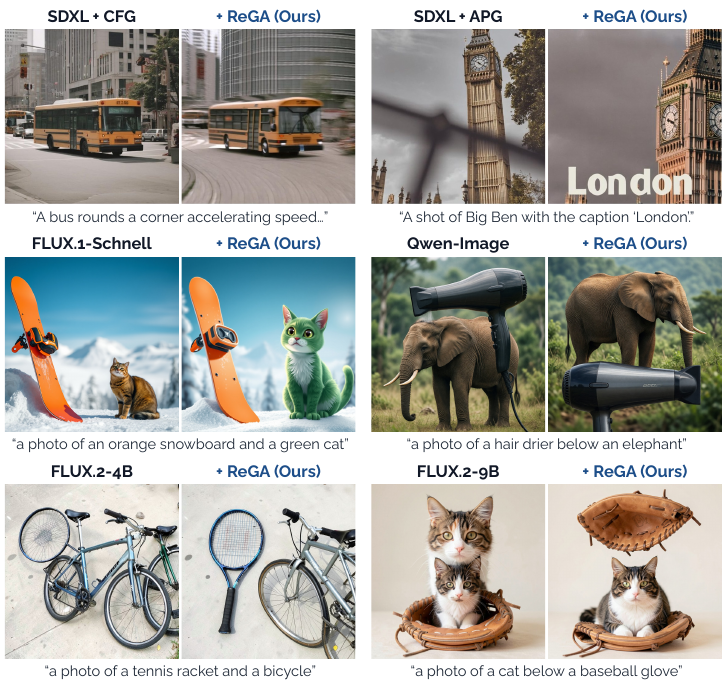}
\vspace{-1.0em}
\captionof{figure}{\textbf{ReGA (Ours)} generalizes across UNet and MM-DiT backbones: SDXL~\citep{sdxl}, FLUX.1~\citep{flux2024}, Qwen-Image~\citep{qwenimage}, and FLUX.2~\citep{flux2}.}
\label{fig:main}
\end{center}
\vspace{-0.8em}
\end{@twocolumnfalse}
}]

\begin{figure*}[!t]
\centering
\includegraphics[scale=1]{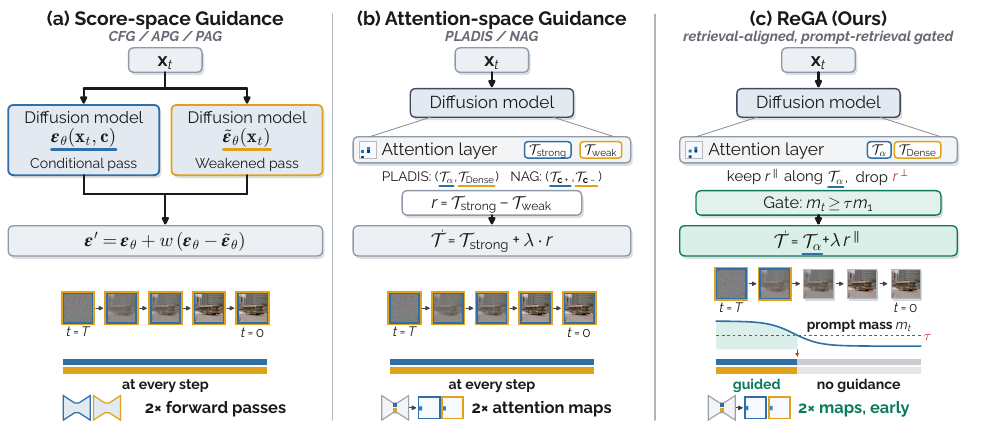}
\caption{\textbf{Conceptual comparison.} (a) Score-space guidance needs two forward passes per step; (b) attention-space guidance is single-pass but extrapolates the full residual at every step; (c) ReGA extrapolates only along the retrieval output, and only over the early steps that hold half of the prompt-attention mass.}
\label{fig:concept}
\vspace{-0.8em}
\end{figure*}

\begin{abstract}
Beyond classifier-free guidance, attention-space guidance extrapolates between a strong and a weak attention map within a single forward pass. However, two questions that govern its behavior remain unexamined: along which direction the extrapolation should proceed, and for how many sampling steps it should be applied. The first governs sample quality, and the second governs cost. Interestingly, viewing attention as sparse Hopfield retrieval answers both questions. At every sampling step, the informative direction is the component of the residual parallel to the sparse retrieval output. The guidance acts on the prompt in proportion to the attention mass that image queries assign to text tokens; and since the prompt shapes the image only in the early part of sampling, guidance matters only while the prompt is still being read. Building on this analysis, we propose \textbf{Re}trieval-\textbf{G}ated \textbf{A}ttention Guidance (ReGA), which extrapolates only along this parallel component and applies the update only over the early steps holding half of the total prompt-attention mass. Extensive measurements inside the attention layers of UNet and MM-DiT backbones confirm both theoretical predictions. Across various backbone models, in multi-step and few-step regimes, ReGA consistently improves text alignment and human-preference scores over existing attention-space guidance methods, while adding as little as 2\% inference time on MM-DiT backbones.
\end{abstract}

\section{Introduction}\label{sec:intro}

Text-to-image (T2I) diffusion models~\citep{rombach2022high,sdxl,stable3} now drive image~\citep{sana,flux2024} and video~\citep{stablevideo,videocrafter2,cogvideox,wan2.1} generation, and because vanilla sampling in diffusion~\citep{ho2020denoising} and flow-matching~\citep{lipman2022flow} models is often suboptimal, guidance is standard practice. Classifier-Free Guidance (CFG)~\citep{CFG} extrapolates between conditional and unconditional score estimates at the cost of two forward passes per step and null-condition training, which hinders its use in few-step distilled or single-step models~\citep{light,turbo,pcm}; training-free alternatives~\citep{PAG,SEG,tpg,asag} remove the null-condition training but still need two passes.

Beyond CFG, attention-space guidance extrapolates between a strong and a weak attention map within a single forward pass: PLADIS~\citep{pladis} contrasts \emph{sparse and dense} attention from the same prompt, and Normalized Attention Guidance (NAG)~\citep{nag} contrasts attention from \emph{positive and negative} prompts. Unlike score-space guidance, whose direction~\citep{apg,adg} and interval~\citep{kynkaanniemi2024applying} have been analyzed, attention-space guidance has not yet been analyzed in either respect: it extrapolates along the full residual between the two maps, although part of it may carry no guidance signal, and applies the update at every step, or stops it after a hand-set fraction of steps~\citep{nag}, although the prompt may no longer shape the sample in later steps. This naturally raises two questions: \emph{which component of the residual is informative?} and \emph{over which sampling steps does guidance remain effective?} Both bear directly on sample quality and cost: extrapolating along an uninformative direction degrades fidelity, and guidance applied after the signal has vanished only adds computation, a cost that weighs most in few-step samplers.

In this work, we focus on the sparse-versus-dense contrast of PLADIS. Its two maps share the queries, keys, values, and conditioning and differ only in how attention probabilities are formed, so their residual isolates the effect of sparsification; contrasts under different conditioning, as in NAG, fall outside this setting. To find out what this residual carries during generation, we measure it inside the attention layers of UNet and MM-DiT backbones, which reveals two regularities (Figure~\ref{fig:motivation}). First, the residual is strongly aligned with the output of sparse attention at every sampling step, and its guidance gain comes from the component parallel to that output, whereas the orthogonal component alone yields no improvement. Second, the prompt is read early: switching it after the first steps no longer changes the image, and on MM-DiT the attention that image queries pay to the text decays over the same steps. Guidance follows this schedule: applied while the prompt is read, it sets the layout, and applied later, it leaves the base sample unchanged (Figure~\ref{fig:x0}).

\begin{figure*}[!t]
\centering
\includegraphics[scale=1]{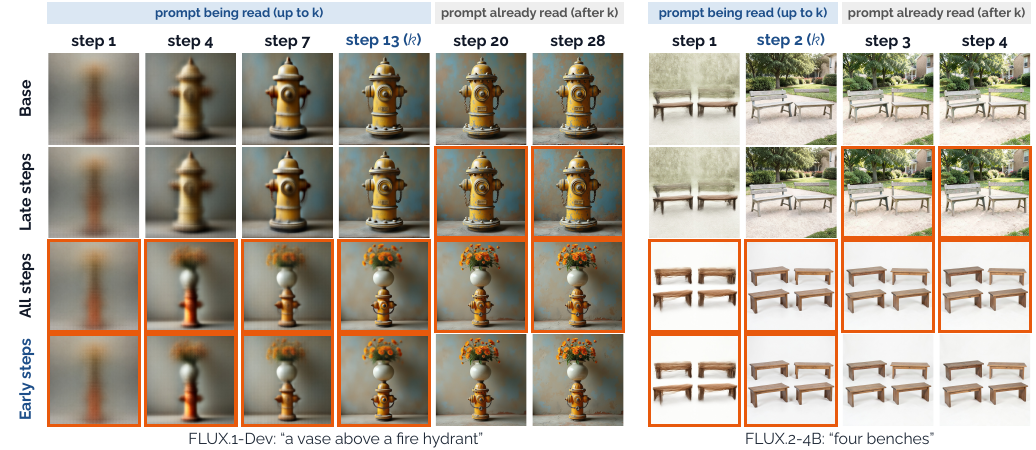}
\vspace{-1.3em}
\caption{\textbf{Guidance within and after the prompt-reading window.} Predicted $\hat{\mbx}_0$ along sampling; orange frames mark the guided steps: none (Base), after $k$ (Late steps), all (All steps), or up to $k$ (Early steps, as in \textbf{ReGA (Ours)}). The layout is decided within the window.}
\label{fig:x0}
\vspace{-0.8em}
\end{figure*}

Interestingly, these regularities find a principled account in the view of attention as sparse Hopfield retrieval~\citep{hopfield,sparsehop,stanhop}. Modeling the sparse and dense maps as two retrieval operators that share queries, keys, and values and differ only in the probability mapping, we show that their discrepancy points toward the retrieved pattern along the sparse retrieval output; since this holds for each attention evaluation, it holds at every sampling step, which accounts for the first regularity. For the second, we show that guidance acts on the prompt in proportion to this attention, which yields a stopping rule: guide only the initial window holding half of it, which ends once the prompt is read (Figure~\ref{fig:x0}).

Building on this analysis, we propose \textbf{Re}trieval-\textbf{G}ated \textbf{A}ttention Guidance (ReGA), which answers both questions: it extrapolates only along the informative component of the residual, the one parallel to the sparse retrieval output, and only within the effective interval, the shortest initial window of steps that holds half of the total prompt-attention mass (Figure~\ref{fig:concept}). As Figure~\ref{fig:main} shows, the effect carries from classical guidance and step-distilled models to the latest MM-DiT backbones. Our contributions are as follows:
\begin{itemize}
\item To identify the informative component of the residual, we propose retrieval-aligned extrapolation, which keeps only the component along the sparse retrieval output; we show that this component is aligned with the retrieval target at every sampling step and that the orthogonal component carries no gain.
\item To determine the effective guidance interval, we propose prompt-retrieval gating, which stops the update once the guided steps hold half of the total prompt-attention mass; we show that the update acts on the prompt in proportion to this mass, so the gate removes computation without removing guidance.
\item Across UNet and MM-DiT backbones, in multi-step and few-step regimes, ReGA consistently improved text alignment and human preference over existing attention-space guidance methods, while adding only a small fraction of the base sampler's inference time on MM-DiT backbones.
\end{itemize}

\section{Preliminary}\label{sec:prelim}

\begin{figure*}[!t]
\centering
\includegraphics[scale=1]{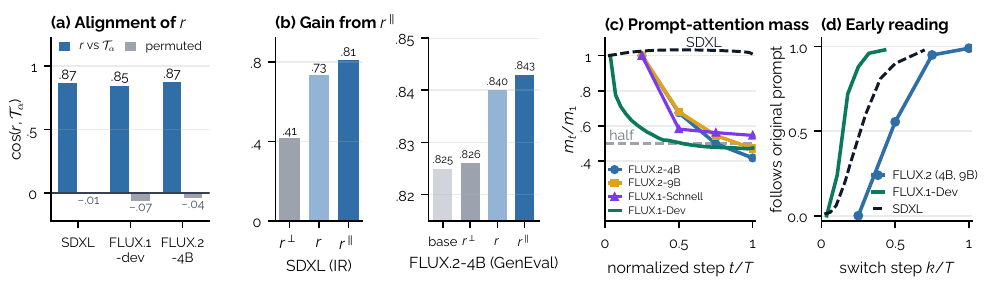}
\vspace{-0.6em}
\caption{\textbf{Attention measured during generation.} (a) The residual is aligned with the retrieval output on every backbone, changing by less than 0.01 across steps; a token-permuted control sits at zero. (b) $r^{\parallel}$ alone outperforms the full residual $r$ and $r^{\perp}$ alone on SDXL and FLUX.2. (c) The prompt-attention mass of the base sampler decays within the first steps and then levels off on MM-DiT, and stays flat on SDXL. (d) Switching the prompt after that point no longer changes the image; on SDXL this point comes near the midpoint.}
\label{fig:motivation}
\vspace{-1.0em}
\end{figure*}

\subsection{Diffusion Models and Guidance Sampling}\label{sec:pre_dm_guide}
Diffusion models~\citep{ho2020denoising, song2021} learn to generate data by reversing a progressive Gaussian noise process. The forward process perturbs $\mbx_0 \sim q(\mbx_0)$ as $q(\mbx_t \mid \mbx_{t-1}) := \mathcal{N}(\mbx_t; \sqrt{1-\beta_t}\mbx_{t-1}, \beta_t \mathbf{I})$, with closed-form marginal $q(\mbx_t) = \mathcal{N}(\mbx_t; \sqrt{\bar{\alpha}_t}\mbx_0, (1-\bar{\alpha}_t)\mathbf{I})$ where $\bar{\alpha}_t = \prod_{i=1}^{t}(1-\beta_i)$. The reverse process $p_\theta$ is trained via denoising score matching~\citep{vincent2011connection}; for conditional generation~\citep{rombach2022high} with context $\mbc$:
\begin{equation}
\min_{\theta} \mathbb{E}_{\mbx_0, \me, \mbc} \left[ \| \me_{\theta}(\sqrt{\bar{\alpha}_t}\mbx_0 + \sqrt{1-\bar{\alpha}_t}\,\me, t, \mbc) - \me \|^2_2 \right].
\end{equation}
Sampling follows the DDIM~\citep{song2021ddim} update with the denoised estimate $\hat{\mbx}_0(t) := (\mbx_t - \sqrt{1-\bar{\alpha}_t}\,\me_\theta(\mbx_t, t, \mbc))/\sqrt{\bar{\alpha}_t}$ from Tweedie's formula~\citep{efron2011tweedie, kim2021noise2score}.

\noindent\textbf{Score-space guidance.} Guidance biases the score toward the conditional distribution by extrapolating the conditional estimate $\me_\theta(\mbx_t,\mbc) := \me_\theta(\mbx_t, t, \mbc)$ away from a weakened counterpart $\tilde{\me}_\theta(\mbx_t)$:
\begin{equation}
\me'_\theta(\mbx_t, \mbc) = \me_\theta(\mbx_t, \mbc) + w\bigl(\me_\theta(\mbx_t, \mbc) - \tilde{\me}_\theta(\mbx_t)\bigr), \label{eq:guidance_unified}
\end{equation}
with guidance scale $w$; anchoring at the conditional estimate matches~\eqref{eq:attn_extra} and equals the CFG convention with $w = w_{\mathrm{CFG}}-1$. CFG~\citep{CFG} uses the unconditional estimate $\me_\theta(\mbx_t, \varnothing)$, weak-model guidance perturbs the architecture or self-attention~\citep{PAG, SEG, asag}, APG~\citep{apg} and ADG~\citep{adg} refine the direction of the contrast, and guidance-interval methods~\citep{kynkaanniemi2024applying} restrict it to a range of noise levels. All require dual forward passes, which limits their use in step-distilled frameworks.

\subsection{Attention-Space Guidance}\label{sec:pre_attn}
To avoid a second forward pass while retaining the extrapolation principle, recent studies~\citep{pladis, nag} apply the same contrast inside the attention layer rather than across forward passes, contrasting a strong attention $\Att$ with a weakened counterpart $\tilde{\Att}$:
\begin{equation}
\Att'(\Qb_t, \Kb_t, \Vb_t) = \Att(\Qb_t, \Kb_t, \Vb_t) + \lambda\bigl(\Att - \tilde{\Att}\bigr)(\Qb_t, \Kb_t, \Vb_t), \label{eq:attn_extra}
\end{equation}
with extrapolation scale $\lambda$. PLADIS~\citep{pladis} uses sparse attention $\Att_\alpha$ (Section~\ref{sec:hop}) as $\Att$ and dense attention as $\tilde{\Att}$, whereas NAG~\citep{nag} uses attention under positive and negative prompts; both apply the contrast at every guided layer and, by default, every sampling step. Along the sampling trajectory, TGATE~\citep{tgate} finds that cross-attention outputs converge within the first few steps and caches them afterwards, from a gate step set per model, to save computation; it targets cost, not guidance. In contrast, we ask which component of the contrast is informative, and we set the guided interval from the attention that the text still receives.

\subsection{Attention as Sparse Hopfield Retrieval}\label{sec:hop}
Attention has been linked to Energy-Based Models through Hopfield energy functions (background in Appendix~\ref{sec:background}): a query $\mbx \in \mathbb{R}^{d}$ is associated with a pattern matrix $\mbXi=[\mbxi_1, \dots, \mbxi_M] \in \mathbb{R}^{d \times M}$ by minimizing an energy $E(\mbx)$ through retrieval dynamics $\mt$, and Modern Hopfield Networks (MHN)~\citep{hopfield} and Sparse Hopfield Networks (SHN)~\citep{sparsehop, stanhop} define
\begin{align}
E(\mbx)_{\texttt{Dense}} &:= -\texttt{lse}(\beta,\mbXi^{\top}\mbx) + \tfrac{1}{2}\langle\mbx,\mbx\rangle, \nonumber\\
\mt_{\texttt{Dense}}(\mbx) &:= \mbXi\,\texttt{Softmax}(\beta\mbXi^{\top}\mbx), \label{eq:dense_dyn}\\
E_{\alpha}(\mbx) &:= -\mathbf{\Psi}^{\star}_{\alpha}(\beta,\mbXi^{\top}\mbx) + \tfrac{1}{2}\langle\mbx,\mbx\rangle, \nonumber\\
\mt_{\alpha}(\mbx) &:= \mbXi\,\alpha\texttt{-Entmax}(\beta\mbXi^{\top}\mbx), \label{eq:sparse_dyn}
\end{align}
where $\texttt{lse}(\beta,\mbz) := \tfrac{1}{\beta}\log\sum_{i=1}^M \exp(\beta z_i)$ and $\mathbf{\Psi}^{\star}_{\alpha}$ is the convex conjugate of the negative Tsallis entropy~\citep{tsallis} $\Psi_{\alpha}$. The $\alpha\texttt{-Entmax}$ operator $\argmax_{\mbp \in \Delta^{M}}[\langle\mbp,\mbz\rangle - \Psi_{\alpha}(\mbp)]$ controls sparsity, recovering softmax at $\alpha=1$ and growing sparser toward $\alpha=2$, with exact, efficient solutions at $\alpha\in\{1.5, 2\}$~\citep{held1974validation,michelot1986finite,entmaxx}; SHN reduces retrieval errors and is more robust to noise than dense dynamics~\citep{sparsehop,pladis,stanhop}. With $\beta = 1/\sqrt{d}$, one application of these dynamics is exactly one attention layer (fixed-point view in Appendix~\ref{sec:background}): $\Att(\Qb, \Kb, \Vb) = \texttt{Softmax}(\Qb\Kb^\top/\sqrt{d})\Vb$ and $\Att_\alpha(\Qb, \Kb, \Vb) = \alpha\texttt{-Entmax}(\Qb\Kb^\top/\sqrt{d})\Vb$, with the query as $\mbx$ and the keys and values as stored patterns; we call $\Tca(\mbx)$, the output of sparse attention $\Att_\alpha$, the \emph{sparse retrieval output}, or simply the retrieval output.

In cross-attention (UNet), every key of an image query is a text token. In joint attention (MM-DiT), the keys of an image query are text and image tokens together, and the image-to-text block holds the attention that image queries assign to text tokens. Following PLADIS, sparse attention and ReGA act on cross-attention and on the image-to-text block of the double-stream layers, since sparsifying self-attention degrades generation~\citep{pladis}.

\section{Motivation}\label{sec:motivation}
We measure the sparse-dense residual and the prompt-attention mass inside the attention layers of SDXL and FLUX during generation (Figure~\ref{fig:motivation}; Appendix~\ref{app:diag}).

\subsection{Direction of the Residual during Generation}\label{sec:observations}
\noindent\textbf{Residual aligned with sparse attention.} We find that the residual $r=\Att_\alpha-\Att$ points largely along the output of sparse attention, $\Att_\alpha$: their cosine similarity is high on all three backbones at every sampling step, whereas a random direction and a token-permuted control give none (Figure~\ref{fig:motivation}a), and on MM-DiT it holds at every double-stream block and head (Figure~\ref{fig:mmdit_diag}). The residual is thus dominated by what sparse attention retrieves, which makes this a meaningful direction for guidance, available at every step at no cost.

\noindent\textbf{Gain from the parallel component.} Since $r$ is aligned with $\Att_\alpha$ but not parallel to it, we next ask which of its parts carries the guidance. We split $r$ into its parallel component $r^{\parallel}$, obtained by projecting the layer's attention residual onto the sparse attention map, and the orthogonal remainder $r^{\perp}$ (Definition~\ref{def:decomp}). We find that the gain comes from $r^{\parallel}$: extrapolating along $r^{\parallel}$ alone matches or exceeds the full residual, whereas $r^{\perp}$ alone degrades SDXL and gives almost no gain on FLUX.2 (Figures~\ref{fig:motivation}b and~\ref{fig:fig_decompo}). The direction of guidance thus matters: the signal lies along $\Att_\alpha$.

\subsection{Prompt Reading during Generation}\label{sec:obs_time}
\noindent\textbf{Prompt-attention mass.} We first measure how much the image reads the text. The \emph{prompt-attention mass}, the total attention that image queries assign to the text block (prompt and padding tokens), decays within the first few steps on MM-DiT backbones and then levels off (Figure~\ref{fig:motivation}c), so guidance acting through the text keys grows weak in later steps. In the cross-attention of SDXL every key is a text token and the share of the prompt tokens stays nearly constant, so there the mass does not show when the reading ends.

\noindent\textbf{Early reading of the prompt.} To see when the prompt is read, we switch it to a different prompt partway through sampling. The final image reflects the new prompt only if the switch occurs in the early steps, and otherwise keeps the content of the original prompt (Figure~\ref{fig:motivation}d); this holds on SDXL as well, where the prompt is about 90\% committed by the midpoint. The prompt is thus read early, and the later steps refine the image without reading it; consistently, guidance applied only after the prompt is read leaves the base sample unchanged (Figure~\ref{fig:x0}).

\noindent\textbf{Open questions.} These observations naturally raise two questions: \emph{why is the component parallel to the sparse retrieval output the informative one?} and \emph{how does the prompt-attention mass govern guidance, and when should it stop?}

\section{Retrieval-Gated Attention Guidance}\label{sec:main}
In short, ReGA extrapolates only along the part of the residual that sparse attention points to, and stops once image tokens no longer attend to the text. Sections~\ref{sec:directional} and~\ref{sec:gate} derive these two rules, and Section~\ref{sec:gag} combines them.

\subsection{Retrieval-Aligned Extrapolation}\label{sec:directional}
We first show why the component of the residual parallel to the sparse retrieval output is the informative one. In the retrieval view of Section~\ref{sec:hop}, the sparse and dense attention outputs are the retrieval operators $\Tca=\Att_\alpha$ and $\Tcd=\Att$, which we use below.

\noindent\textbf{Setup.} Let $\Tca$ and $\Tcd$ be instantiated from the same model with identical $(\Qb, \Kb, \Vb)$ and conditioning $\mbc$, differing only in the probability mapping ($\alpha\texttt{-Entmax}$ vs.\ \texttt{Softmax}). Both outputs lie in $\mathrm{span}(\Vb)$, and we define the discrepancy
\begin{equation}
r(\mbx) \;:=\; \Tca(\mbx) - \Tcd(\mbx) \;=\; \Vb\bigl(\mbp_\alpha(\mbx) - \mbp_{\texttt{dense}}(\mbx)\bigr), \label{eq:rdef}
\end{equation}
where $\mbp_\alpha, \mbp_{\texttt{dense}}\in\Delta^M$.

\begin{figure}[!t]
\centering
\includegraphics[scale=1]{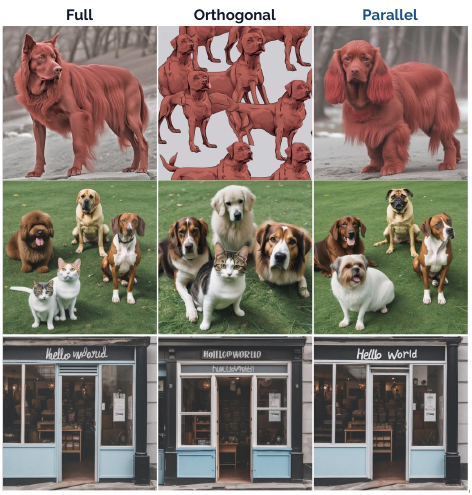}
\vspace{-0.7em}
\caption{\textbf{Guiding along each component of} $r$ on SDXL: $r^{\perp}$ alone fails, the full residual interferes, and $r^{\parallel}$ is best.}
\label{fig:fig_decompo}
\vspace{0.1em}
\end{figure}

\begin{theorem}\label{thm:directional}
Under the setup above, let $\xib_\mu$ be the dominant stored pattern, let $\Delta(\mbx) := \langle\xib_\mu,\mbx\rangle - \max_{j\ne\mu}\langle\xib_j,\mbx\rangle > 0$ be its logit margin at $\mbx$, and $\alpha\in(1,2]$. Then:
\begin{enumerate}
\item[(i)] \textbf{Convergence rate separation.} The retrieval error of sparse retrieval is no larger than that of dense retrieval on the storage ball~\citep{sparsehop, stanhop}, and for $\beta \ge \beta_0(\mbx) := 1/((\alpha-1)\Delta(\mbx))$ sparse retrieval is exact, $\Tca(\mbx) = \xib_\mu$, so that
\begin{equation}
\norm{\Tca(\mbx) - \xib_\mu} \;=\; 0\;\le\; \norm{\Tcd(\mbx) - \xib_\mu}. \label{eq:rate_sep}
\end{equation}
\item[(ii)] \textbf{Directional alignment.} At every $\mbx$,
\begin{equation}
\bigl\langle r(\mbx),\, \xib_\mu - \Tcd(\mbx) \bigr\rangle \;\ge\; \tfrac{1}{2}\Bigl(\norm{r(\mbx)}^2 - \norm{\Tca(\mbx) - \xib_\mu}^2\Bigr). \label{eq:dir_align}
\end{equation}
Hence $r(\mbx)$ is non-negatively aligned with the residual $\xib_\mu - \Tcd(\mbx)$ whenever $\norm{r(\mbx)} \ge \norm{\Tca(\mbx) - \xib_\mu}$; above the threshold the sparse error vanishes and $r(\mbx) = \xib_\mu - \Tcd(\mbx)$.
\end{enumerate}
\end{theorem}

Theorem~\ref{thm:directional} accounts for the first observation of Section~\ref{sec:observations}: the residual points toward the retrieved pattern, and since the theorem concerns a single attention evaluation, it holds at every sampling step, which changes only the magnitude of $r$. It requires the two operators to share $(\Qb,\Kb,\Vb)$ and $\mbc$; NAG contrasts attention under different conditioning and lies outside it, and composing ReGA with NAG degrades accordingly (Table~\ref{tab:nag_full}).

\pagebreak
\begin{definition}\label{def:decomp}
For the attention tensors $\mathbf{P}_\alpha, \mathbf{P}_{\texttt{dense}}$ of a layer and $\boldsymbol{\Delta} := \mathbf{P}_\alpha - \mathbf{P}_{\texttt{dense}}$, let $\boldsymbol{\Delta}^{\parallel}$ be the projection of $\boldsymbol{\Delta}$ onto $\mathbf{P}_\alpha$ in the Frobenius inner product and $\boldsymbol{\Delta}^{\perp} := \boldsymbol{\Delta} - \boldsymbol{\Delta}^{\parallel}$; $r^{\parallel}(\mbx)$ and $r^{\perp}(\mbx)$ are the rows of $\Vb\boldsymbol{\Delta}^{\parallel}$ and $\Vb\boldsymbol{\Delta}^{\perp}$ at $\mbx$.
\end{definition}

\begin{corollary}\label{cor:parallel}
Under the setup of Theorem~\ref{thm:directional}, assume $\norm{\mbp_\alpha(\mbx)}\ge\norm{\mbp_{\texttt{dense}}(\mbx)}$, $\beta\ge\beta_0(\mbx)$, and $\langle\xib_\mu,\xib_j\rangle\le\norm{\xib_\mu}^2$ for all $j$. Then
\begin{equation}
\bigl\langle r^{\parallel}(\mbx),\, \xib_\mu-\Tcd(\mbx)\bigr\rangle \;\ge\; 0 \label{eq:par_align}
\end{equation}
\noindent(proof in Appendix~\ref{app:par_proof}; fixed points in~\ref{app:cor_proof}).
\end{corollary}

\noindent\textbf{Suppressing the orthogonal component.} Corollary~\ref{cor:parallel} accounts for the second observation: the alignment is carried by $r^{\parallel}$, whereas $r^{\perp}$ carries no such guarantee and, on SDXL, harms through its direction rather than its scale (Proposition~\ref{thm:conv_orth} and Figure~\ref{fig:residual} in Appendix~\ref{app:diag}). On SDXL, guiding along $r^{\parallel}$ alone gives the best samples while $r^{\perp}$ alone fails (Figure~\ref{fig:fig_decompo}). We therefore extrapolate along $r^{\parallel}$ only, which answers the first question.

\subsection{Prompt-Retrieval Gating}\label{sec:gate}
Section~\ref{sec:directional} fixes the direction of the update; we now ask how much of it acts on the prompt, which tells when to stop. In joint attention, $r^{\parallel}$ reads the values of both text and image keys, and only the text part carries the prompt. For an image query $\mbx$ whose text keys form the set $S$ (the prompt and padding tokens in joint attention; the prompt tokens in cross-attention), let $m(\mbx):=\sum_{j\in S}p_j(\mbx)$ be the \emph{prompt-attention mass} of an attention row $\mbp(\mbx)$, $m_\alpha$ that of the sparse row $\mbp_\alpha$, and $m_t$ its mean over guided layers, heads, queries, and prompts at step $t$ of the base sampler.

\begin{proposition}\label{prop:share}
Let $\boldsymbol{\delta}(\mbx)$ be the row of $\boldsymbol{\Delta}^{\parallel}$ at an image query $\mbx$, and let $\bar{\mbv}_S(\mbx)$ and $\bar{\mbv}_I(\mbx)$ be the readouts of the text and image values by the sparse row $\mbp_\alpha(\mbx)$, each normalized by its mass. Then, omitting $\mbx$,
\begin{equation}
r^{\parallel} = \Bigl(\textstyle\sum_j \delta_j\Bigr)\bigl(m_\alpha\,\bar{\mbv}_S + (1-m_\alpha)\,\bar{\mbv}_I\bigr) \label{eq:share}
\end{equation}
\noindent(proof in Appendix~\ref{app:share_proof}).
\end{proposition}

By Proposition~\ref{prop:share}, the update is a readout of the prompt weighted by $m_\alpha$ plus a readout of the image weighted by $1-m_\alpha$, so as the mass decays, guidance loses its hold on the prompt, just as the prompt stops shaping the image. We therefore guide the shortest initial window that retains half of the prompt-directed guidance of guiding every step,
\begin{equation}
k:=\min\Bigl\{k' : \textstyle\sum_{t\le k'} m_t \;\ge\; \tfrac{1}{2}\sum_{t=1}^{T} m_t\Bigr\}, \label{eq:window}
\end{equation}
which answers the second question. We read $m_t$ once per backbone from the dense attention of the base sampler, which is available before any guidance is applied, so the window is fixed in advance, and an online gate that follows the actual trajectory gives the same result (Section~\ref{sec:analysis}). Because the rule fixes a share of the mass rather than a threshold on it, it applies to every backbone: where the mass decays, as on MM-DiT, the window ends within the first steps, and where it stays flat, as on SDXL, it covers the first half of sampling, by which point the prompt is about 90\% committed (Figure~\ref{fig:motivation}d).

\begin{algorithm}[t]
\caption{Sampling with ReGA}\label{alg:rega}
\small
\begin{algorithmic}[1]
\Statex \textbf{Input:} model $f_\theta$ with attention layers $\Att^{(\ell)}$, $\ell=1,\dots,L$; scale $\lambda$; window $k$ (Eq.~\eqref{eq:window})
\State $\mbx_T \sim \mathcal{N}(\mathbf{0},\mathbf{I})$
\For{$i=1,\dots,T$ \textbf{with} $t=T-i+1$}
\If{$i\le k$}
\State compute each $\Att^{(\ell)}$ by Eq.~\eqref{eq:gag}
\Else
\State compute each $\Att^{(\ell)}$ as the base attention
\EndIf
\State $\mbx_{t-1}\gets\textsc{Update}\bigl(\mbx_t,\ f_\theta(\mbx_t,\mbc)\bigr)$ \Comment{e.g., DDIM, Euler}
\EndFor
\State \Return $\mbx_0$
\end{algorithmic}
\end{algorithm}

\subsection{Our Method: ReGA}\label{sec:gag}
Combining the two components, ReGA computes the output of a guided layer at the $t$-th sampling step as
\begin{equation}
\mt^{\lambda}(\mbx) = \begin{cases}\Tca(\mbx) + \lambda\, s\, r^{\parallel}(\mbx), & t\le k,\\ \Att(\mbx), & t>k,\end{cases} \label{eq:gag}
\end{equation}
where $\lambda$ is the extrapolation scale, $s := \min(1,\eta/\norm{\boldsymbol{\Delta}}_F)$ clamps the norm of the layer's discrepancy and keeps the update stable (Appendix~\ref{sec:stable}), and $k$ is the window of Eq.~\eqref{eq:window}. Both attention maps come from the same forward pass, so ReGA needs no extra network evaluation; since $k$ is fixed in advance, the gate adds no computation, and sparse attention is evaluated only inside the window. Algorithm~\ref{alg:rega} summarizes sampling with ReGA.

\section{Experiments}\label{sec:exp}
\noindent\textbf{Models.} To test ReGA across architectures and regimes, we apply it to SDXL~\citep{sdxl} under CFG, APG~\citep{apg}, and CFG+PAG at 50 steps, to the step-distilled SDXL-DMD2~\citep{dmd2}, Hyper-SDXL~\citep{hyper}, and SDXL-Lightning~\citep{light} at 4 steps, and to the MM-DiT backbones FLUX.1-Schnell and FLUX.1-Dev~\citep{flux2024}, FLUX.2 (4B, 9B)~\citep{flux2}, and Qwen-Image~\citep{qwenimage}.

\noindent\textbf{Evaluation.} We measure text alignment with GenEval (GE)~\citep{geneval} and CLIPScore (CS)~\citep{clipscore}, image quality with FID on MS-COCO~\citep{coco}, and human preference with ImageReward (IR)~\citep{imagereward}, PickScore (PS)~\citep{pick}, and HPSv2~\citep{hpsv2} on $10$K MS-COCO captions, and complement them with a blind user study (Appendix~\ref{app:userstudy}).

\noindent\textbf{Implementation.} ReGA guides every cross-attention layer on SDXL and the image-to-text block of the double-stream layers on MM-DiT, with $\lambda{=}10$ and $5$, respectively (protocol in Appendix~\ref{sec:detail}). The window is fixed once per backbone from the base sampler (Section~\ref{sec:gate}): steps 1--2 on FLUX.2, FLUX.1-Schnell, and Qwen-Image and 1--13 of 28 on FLUX.1-Dev; and 1--26 of 50 on SDXL (Appendix~\ref{app:settings}).

\subsection{Main Results}\label{sec:main_results}
\noindent\textbf{UNet backbones.} On SDXL with CFG, ReGA improves every metric over both the base sampler and PLADIS and lifts GenEval above $0.6$, a level that SDXL does not reach with guidance alone (Table~\ref{tab_main}). The gain persists when the base guidance changes: with APG, a variant of CFG, and with CFG+PAG, which combines two guidance signals, ReGA again improves alignment and most preference metrics over both baselines. Step-distilled models are where dual-pass guidance such as CFG or APG cannot be applied, and there ReGA ranks first on nearly every metric across DMD2, Hyper-SDXL, and SDXL-Lightning, so the single-pass update works robustly across distillation methods without retuning.

\noindent\textbf{MM-DiT backbones.} ReGA carries over to joint attention without any training. On the guidance-distilled FLUX.1-Schnell and FLUX.1-Dev, it improves GenEval well beyond PLADIS, while improving human preference (Table~\ref{tab_main_flux}). On FLUX.2, where PLADIS barely moves GenEval at 4B, ReGA gives a clear gain at both scales, and the gain carries over to Qwen-Image. These gains are not marginal: across five MM-DiT backbones that differ in size, training, and distillation, the same update improves alignment every time, which shows that ReGA generalizes beyond the UNet on which attention-space guidance was developed. The gate keeps this gain at lower cost (Section~\ref{sec:analysis}).

\begin{table}[t]
\centering
\caption{Quantitative comparison on SDXL and its step-distilled models. \textbf{Bold}: best within each group.}
\label{tab_main}
\setlength{\tabcolsep}{2.6pt}
\footnotesize
\resizebox{0.98\columnwidth}{!}{\begin{tabular}{lccccccc}
\toprule
Method & Steps & GE~$\uparrow$ & FID~$\downarrow$ & CS~$\uparrow$ & IR~$\uparrow$ & PS~$\uparrow$ & HPSv2~$\uparrow$ \\
\cmidrule(lr){1-1}\cmidrule(lr){2-2}\cmidrule(lr){3-3}\cmidrule(lr){4-8}
\multicolumn{8}{l}{\emph{SDXL}}\\
CFG & 50 & 0.545 & 20.22 & 26.45 & 0.604 & 22.31 & 0.266 \\
\midrule
+\,PLADIS & 50 & 0.596 & 15.81 & 26.97 & 0.799 & 22.32 & 0.268 \\
\rowcolor{green!10} +\,\textbf{ReGA (Ours)} & 50 & \bf 0.619 & \bf 14.65 & \bf 27.44 & \bf 0.822 & \bf 22.39 & \bf 0.271 \\
\midrule
APG & 50 & 0.557 & 18.53 & 26.78 & 0.702 & \bf 22.36 & 0.269 \\
\midrule
+\,PLADIS & 50 & 0.608 & 14.51 & 27.36 & 0.808 & 21.92 & 0.261 \\
\rowcolor{green!10} +\,\textbf{ReGA (Ours)} & 50 & \bf 0.615 & \bf 13.24 & \bf 27.54 & \bf 0.848 & 22.20 & \bf 0.271 \\
\midrule
CFG+PAG & 100 & 0.528 & 19.59 & 25.98 & 0.640 & 22.30 & 0.277 \\
\midrule
+\,PLADIS & 100 & 0.598 & 15.43 & 26.05 & 0.751 & 22.33 & 0.279 \\
\rowcolor{green!10} +\,\textbf{ReGA (Ours)} & 100 & \bf 0.617 & \bf 14.40 & \bf 26.23 & \bf 0.824 & \bf 22.36 & \bf 0.283 \\
\midrule
\multicolumn{8}{l}{\emph{Step-distilled SDXL}}\\
SDXL-DMD2 & 4 & 0.581 & 19.83 & 26.69 & 0.888 & 22.64 & 0.297 \\
\midrule
+\,PLADIS & 4 & 0.585 & 18.78 & 27.06 & \bf 0.965 & 22.65 & 0.297 \\
\rowcolor{green!10} +\,\textbf{ReGA (Ours)} & 4 & \bf 0.598 & \bf 17.49 & \bf 27.42 & 0.952 & \bf 22.66 & \bf 0.298 \\
\midrule
Hyper-SDXL & 4 & 0.556 & 21.23 & 26.09 & 0.959 & 22.87 & 0.313 \\
\midrule
+\,PLADIS & 4 & 0.589 & 19.83 & 27.11 & 1.060 & 22.91 & 0.315 \\
\rowcolor{green!10} +\,\textbf{ReGA (Ours)} & 4 & \bf 0.594 & \bf 19.12 & \bf 27.55 & \bf 1.071 & \bf 22.93 & \bf 0.317 \\
\midrule
SDXL-Light & 4 & 0.535 & 24.42 & 26.05 & 0.725 & 22.62 & 0.293 \\
\midrule
+\,PLADIS & 4 & 0.570 & 23.23 & 26.70 & 0.835 & 22.62 & 0.290 \\
\rowcolor{green!10} +\,\textbf{ReGA (Ours)} & 4 & \bf 0.579 & \bf 22.85 & \bf 26.79 & \bf 0.843 & \bf 22.66 & \bf 0.294 \\
\bottomrule
\end{tabular}}
\vspace{-0.4em}
\end{table}

\begin{table}[t]
\centering
\caption{Quantitative comparison on MM-DiT backbones. \textbf{Bold}: best per backbone. $^\dagger$Qwen-Image-Lightning, distilled to 4 steps.}
\label{tab_main_flux}
\setlength{\tabcolsep}{3pt}
\small
\resizebox{0.98\columnwidth}{!}{\begin{tabular}{lcccccc}
\toprule
Method & Steps & GE~$\uparrow$ & CS~$\uparrow$ & IR~$\uparrow$ & PS~$\uparrow$ & HPSv2~$\uparrow$ \\
\cmidrule(lr){1-1}\cmidrule(lr){2-2}\cmidrule(lr){3-3}\cmidrule(lr){4-7}
FLUX.1-Schnell & 4 & 0.669 & 26.28 & 1.059 & 22.62 & 0.295 \\
\midrule
+\,PLADIS & 4 & 0.712 & 26.45 & 1.066 & 22.31 & 0.295 \\
\rowcolor{green!10} +\,\textbf{ReGA (Ours)} & 4 & \bf 0.735 & \bf 26.66 & \bf 1.135 & \bf 22.86 & \bf 0.299 \\
\midrule
FLUX.1-Dev & 28 & 0.666 & 25.71 & 1.093 & 23.00 & 0.304 \\
\midrule
+\,PLADIS & 28 & 0.691 & 25.80 & 1.115 & 23.03 & 0.305 \\
\rowcolor{green!10} +\,\textbf{ReGA (Ours)} & 28 & \bf 0.731 & \bf 25.96 & \bf 1.156 & \bf 23.08 & \bf 0.307 \\
\midrule
Qwen-Image$^\dagger$ & 4 & 0.824 & 27.63 & 1.123 & 23.13 & 0.309 \\
\midrule
+\,PLADIS & 4 & 0.828 & 27.63 & 1.156 & 23.15 & 0.310 \\
\rowcolor{green!10} +\,\textbf{ReGA (Ours)} & 4 & \bf 0.841 & \bf 28.01 & \bf 1.167 & \bf 23.17 & \bf 0.311 \\
\midrule
FLUX.2-4B & 4 & 0.824 & 27.53 & 1.156 & 23.01 & 0.311 \\
\midrule
+\,PLADIS & 4 & 0.826 & 27.52 & 1.168 & 23.03 & 0.311 \\
\rowcolor{green!10} +\,\textbf{ReGA (Ours)} & 4 & \bf 0.843 & \bf 27.71 & \bf 1.181 & \bf 23.11 & \bf 0.313 \\
\midrule
FLUX.2-9B & 4 & 0.849 & 27.68 & 1.213 & 23.15 & 0.315 \\
\midrule
+\,PLADIS & 4 & 0.865 & 27.81 & 1.223 & 23.14 & 0.313 \\
\rowcolor{green!10} +\,\textbf{ReGA (Ours)} & 4 & \bf 0.877 & \bf 28.00 & \bf 1.251 & \bf 23.22 & \bf 0.317 \\
\bottomrule
\end{tabular}}
\vspace{-0.4em}
\end{table}

\noindent\textbf{Qualitative results.} Figure~\ref{fig:main} and Appendix~\ref{sec:visual} show the same pattern visually: on every backbone, ReGA repairs the failures of the base sampler in counting, spatial relations, attribute binding, and rendered text while keeping the layout and style of the base sample.

\section{Analysis}\label{sec:analysis}
\noindent\textbf{Component analysis.} To check that each design choice of Section~\ref{sec:main} is needed, Table~\ref{tab:ablation} adds the components of ReGA one at a time on both architectures. Extrapolating the full residual, as PLADIS does, raises GenEval on SDXL but barely moves it on FLUX.2-4B; keeping only the parallel component raises it on both, and the norm clamp gives the largest step on FLUX.2-4B, where guiding along the orthogonal component alone stays at the base level. The clamp also protects the preference metrics on both backbones (Appendix~\ref{app:ablations}). The gain thus comes from the parallel direction, as Theorem~\ref{thm:directional} predicts, and normalization protects quality.

\begin{table}[t]
\centering
\caption{\textbf{Component ablation} (GenEval). Components are added from top to bottom.}
\label{tab:ablation}
\footnotesize
\setlength{\tabcolsep}{4pt}
\begin{tabular}{lcc}
\toprule
Configuration & SDXL (CFG) & FLUX.2-4B \\
\midrule
Base & 0.545 & 0.824 \\
+\,full residual & 0.594 & 0.826 \\
+\,parallel only & 0.601 & 0.831 \\
+\,norm clamp (w/o gate) & 0.607 & \bf 0.843 \\
\rowcolor{green!10} +\,gate (\textbf{ReGA}) & \bf 0.619 & \bf 0.843 \\
\bottomrule
\end{tabular}
\vspace{-1.1em}
\end{table}

\begin{figure}[t]
\centering
\includegraphics[scale=1]{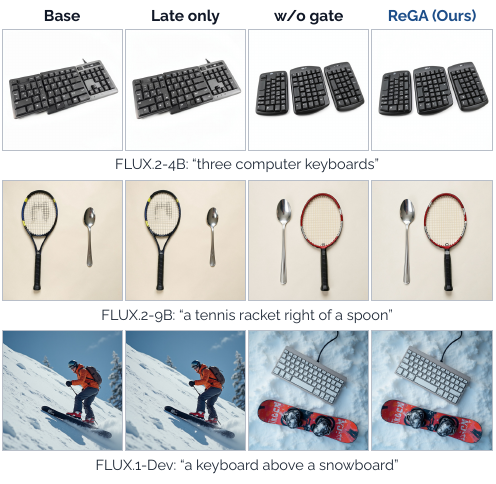}
\vspace{-1.45em}
\caption{\textbf{Qualitative effect of the gate.} Prompts on which the base sampler fails: late-step guidance leaves the base result, whereas ReGA corrects it as guiding every step does.}
\label{fig:qual_gate}
\vspace{-0.4em}
\end{figure}

\begin{figure}[t]
\centering
\includegraphics[scale=1]{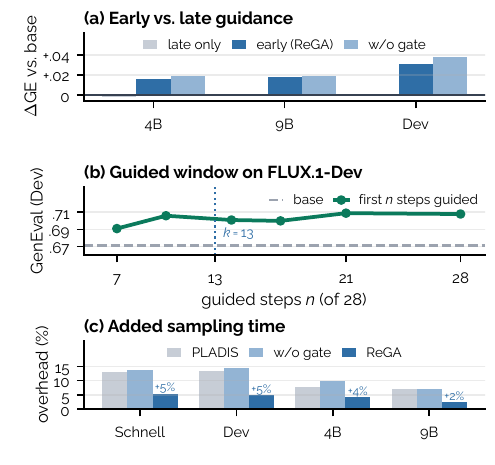}
\vspace{-1.4em}
\caption{\textbf{Effect of the gate.} (a) GenEval gain of late-only, early (ReGA), and ungated guidance. (b) FLUX.1-Dev GenEval against guided steps. (c) Sampling time added to the base sampler (H100).}
\label{fig:gate}
\vspace{-0.4em}
\end{figure}

\noindent\textbf{Effect of the gate.} To verify that the gate removes only steps in which guidance no longer acts, we examine it from\pagebreak{} three complementary views. First, splitting the steps at $k$, guidance restricted to the steps after $k$ yields no gain on any backbone, whereas guidance restricted to the steps up to $k$ recovers nearly the full gain of guiding every step (Figure~\ref{fig:gate}a). Second, on FLUX.1-Dev, GenEval saturates once the guided steps reach $k$ (Figure~\ref{fig:gate}b). Third, the gate preserves the outcome of the full update: ReGA matches or exceeds the ungated update on SDXL and FLUX.2-4B (Table~\ref{tab:ablation}), and it corrects the same failures of the base sampler as guiding every step, which late-step guidance does not (Figure~\ref{fig:qual_gate}). Guidance is therefore effective only while the prompt is read, and the gate retains its full gain while guiding at most about half of the steps.

\noindent\textbf{Inference cost.} Because sparse attention runs only inside the window, ReGA comes almost for free: it adds only a small fraction to the sampling time on MM-DiT with peak memory unchanged, less than PLADIS on every backbone including SDXL, and removes more than half of the overhead of guiding every step (Figure~\ref{fig:gate}c; Appendix~\ref{app:overhead}).

\noindent\textbf{Analysis on prompt categories.} To see where the gain comes from, we break GenEval down by category (Table~\ref{tab:cat_main} in Appendix~\ref{app:cat}). The gain lies in prompts that bind several entities, counting, position, and attribute binding, while single-object prompts are saturated: there each image region must read the right token among several, which a sharper retrieval provides.

\begin{table}[t]
\centering
\caption{Three-way user study (95\% CIs).}
\label{tab:userstudy}
\vspace{-0.8em}
\footnotesize
\setlength{\tabcolsep}{4pt}
\begin{tabular}{lcc}
\toprule
Method & Text Align. & Visual Quality \\
\midrule
Baseline & $20.3\%$ $[18.2{-}23.2]$ & $23.2\%$ $[21.1{-}24.1]$ \\
PLADIS & $22.5\%$ $[20.1{-}24.1]$ & $24.3\%$ $[22.3{-}26.1]$ \\
\rowcolor{green!10} ReGA & $\mathbf{57.2\%}$ $[54.1{-}60.1]$ & $\mathbf{52.5\%}$ $[49.8{-}55.1]$ \\
\bottomrule
\end{tabular}
\vspace{-0.6em}
\end{table}

\noindent\textbf{User study.} Metrics can disagree with people, so we asked $35$ participants to compare base, PLADIS, and ReGA on $50$ prompts. ReGA is preferred for both text alignment and visual quality, with intervals that do not overlap either baseline (Table~\ref{tab:userstudy}; Appendix~\ref{app:userstudy}).

\noindent\textbf{Hyperparameter study.} ReGA is insensitive to $\lambda$ and $\eta$ over wide ranges on SDXL and FLUX.2-4B; the sweeps are in Appendix~\ref{app:hparam} and a comparison with NAG in Appendix~\ref{app:nag}.

\section{Conclusion}\label{sec:concl}
In this work, we revisited attention-space guidance through the lens of sparse Hopfield retrieval and proposed ReGA, which guides only along the component of the residual parallel to the sparse retrieval output and only while the prompt is being read. Across UNet and MM-DiT backbones, in multi-step and few-step regimes, ReGA improves text alignment and human preference over existing attention-space guidance at a small fraction of the sampling cost. We believe that reading attention as retrieval offers a principled way to decide where and when to guide, and we hope it carries over to video and other multimodal diffusion models.

{\small
\bibliographystyle{ieeenat_fullname}
\bibliography{refer}
}

\clearpage
\appendix
\onecolumn
\begin{center}
{\Large\bf Guidance Ends When Reading Ends: Retrieval-Gated Attention Guidance for\\Diffusion Models through a Hopfield Lens\par}
\vspace{0.4em}
{\Large Supplementary Material\par}
\end{center}
\vspace{0.6em}

\newcommand{\apx}[3]{{\small\noindent\hyperref[#1]{\makebox[16pt][l]{\textbf{\ref*{#1}}}\textbf{#2}}\dotfill\pageref{#1}\par}\vspace{-1pt}{\footnotesize\noindent\hspace*{16pt}#3\par}\vspace{4pt}}
\noindent{\large\bf Contents}\par\vspace{6pt}
\apx{sec:deferred}{Theory}{Hopfield background; proofs of Theorem~\ref{thm:directional}, Corollary~\ref{cor:parallel}, and Proposition~\ref{prop:share}; fixed-point behavior; stability.}
\apx{app:impl}{Implementation and Settings}{Sampling protocol, per-backbone settings and guided windows, implementation and cost.}
\apx{app:results}{Additional Results}{Gate, GenEval per category, hyperparameters, component ablations, comparison with NAG, SDXL and MM-DiT diagnostics, DDIM inversion.}
\apx{app:userstudy}{User Study Protocol}{Design of the blind three-way study.}
\apx{sec:visual}{Additional Qualitative Results}{Samples for every guidance method and backbone.}
\apx{app:discussion}{Limitations and Discussion}{Limitations, model-level versus layer-level Hopfield views, scope of the image-to-text block.}
{\small\noindent\hspace*{16pt}\hyperref[app:repro]{\textbf{Reproducibility and Ethics}}\dotfill\pageref{app:repro}\par}
\vspace{0.8em}

\let\oldappsection\section\renewcommand{\section}{\FloatBarrier\oldappsection}

\section{Theory}\label{sec:deferred}

\subsection{Background on (Sparse) Modern Hopfield Networks}\label{sec:background}\label{sec:fixed}
\noindent\textbf{Notations.} For $a \in \mathbb{R}$, $a_+ := \max\{0, a\}$. For $\mathbf{z}, \mathbf{z}' \in \mathbb{R}^d$, $\langle \mathbf{z}, \mathbf{z}'\rangle = \mathbf{z}^\top \mathbf{z}'$ denotes the inner product. For $\mathbf{z} = (z_1,\dots,z_d) \in \mathbb{R}^d$, the sorted coordinates are written $z_{(1)} \ge z_{(2)} \ge \dots \ge z_{(d)}$, where $z_{(\nu)}$ is the $\nu$-th largest entry. We write $\Delta^{M-1} := \{\mathbf{p} \in \mathbb{R}^M : p_i \ge 0,\, \sum_i p_i = 1\}$ for the $(M-1)$-dimensional simplex.

\noindent\textbf{Modern Hopfield Networks.} The classical Hopfield model~\citep{hopfield1982neural} stores binary patterns as local minima of an associated energy function and retrieves the closest minimum given a query input; numerous extensions have been proposed to improve stability and capacity~\citep{demircigil2017model, krotov2016dense, barra2018new}. \citet{hopfield} introduced the \emph{Modern Hopfield Network} (MHN), a continuous-state, differentiable variant equipped with an energy function $E$ and retrieval dynamics $\mathcal{T}$ that typically retrieve well-separated patterns in a single update:
\begin{align}
E_{\texttt{Dense}}(\mathbf{x}) &:= -\texttt{lse}(\beta, \mathbf{\Xi}^\top \mathbf{x}) + \tfrac{1}{2}\langle \mathbf{x}, \mathbf{x}\rangle, \\
\mathcal{T}_{\texttt{Dense}}(\mathbf{x}) &:= \mathbf{\Xi}\,\texttt{Softmax}(\beta\,\mathbf{\Xi}^\top \mathbf{x}), \label{eq:app_dense}
\end{align}
where $\mathbf{x} \in \mathbb{R}^d$ is the query, $\mathbf{\Xi} = [\xib_1, \dots, \xib_M] \in \mathbb{R}^{d \times M}$ collects the stored patterns $\xib_i \in \mathbb{R}^d$, and $\texttt{lse}(\beta, \mathbf{z}) := \tfrac{1}{\beta}\log\sum_i \exp(\beta z_i)$ is the log-sum-exponential function. Convergence and state properties of the energy and retrieval dynamics are established in~\citep{hopfield}.

\noindent\textbf{Connection with Transformer attention.} A key observation of \citet{hopfield} is that self-attention coincides with one step of the MHN update rule. Following~\citep{hopfield}, $\mathcal{T}_{\texttt{Dense}}$ is first extended to multiple queries $\mathbf{X} := \{\mathbf{x}_i\}_{i\in[N]}$. With raw query $\mathbf{R}$ and memory matrix $\mathbf{Y}$, define the query/key matrices via $\mathbf{X}^\top = \mathbf{R}\mathbf{W}_Q := \mathbf{Q}$ and $\mathbf{\Xi}^\top = \mathbf{Y}\mathbf{W}_K := \mathbf{K}$. Taking the transpose and projecting $\mathbf{K}$ to $\mathbf{V}$ via $\mathbf{W}_V$ in~\eqref{eq:app_dense}, we obtain
\begin{equation}
\mathcal{T}_{\texttt{Dense}}(\mathbf{X}) = \texttt{Softmax}(\beta\,\mathbf{Q}\mathbf{K}^\top)\,\mathbf{K}\mathbf{W}_V = \texttt{Softmax}(\beta\,\mathbf{Q}\mathbf{K}^\top)\,\mathbf{V}, \label{eq:app_at1}
\end{equation}
which corresponds to standard scaled dot-product attention with $\beta = 1/\sqrt{d}$. Adopting the notation of~\citep{hopfield},
\begin{equation}
\mathcal{T}_{\texttt{Dense}}(\mathbf{X}) = \texttt{Softmax}(\mathbf{Q}\mathbf{K}^\top/\sqrt{d})\,\mathbf{V} := \texttt{At}(\mathbf{Q}, \mathbf{K}, \mathbf{V}). \label{eq:app_at2}
\end{equation}
This interpretation extends naturally to cross-attention with inputs $\mathbf{X}$ and $\mathbf{Y}$:
\[
\mathcal{T}_{\texttt{Dense}}(\mathbf{X}, \mathbf{Y}) = \texttt{Softmax}\bigl(\mathbf{X} \mathbf{W}_Q \mathbf{W}_K^\top \mathbf{Y}^\top / \sqrt{d}\bigr) \mathbf{Y} \mathbf{W}_V = \texttt{At}(\mathbf{W}_Q\mathbf{X},\, \mathbf{W}_K\mathbf{Y},\, \mathbf{W}_V\mathbf{Y}).
\]
Our analysis focuses on this setting, in which text tokens act as stored memory patterns and image tokens as queries. Since different layers may store distinct patterns, MHNs naturally facilitate hierarchical storage and retrieval.

\noindent\textbf{Sparse Hopfield Network and $\alpha$-Entmax.} A sparse extension of MHN replaces the softmax mapping with sparsemax~\citep{sparsehop}, generalized to $\alpha\texttt{-Entmax}$~\citep{stanhop}:
\begin{align}
E_{\alpha}(\mathbf{x}) &:= -\mathbf{\Psi}^\star_{\alpha}(\beta, \mathbf{\Xi}^\top \mathbf{x}) + \tfrac{1}{2}\langle \mathbf{x}, \mathbf{x}\rangle, \\
\mathcal{T}_{\alpha}(\mathbf{x}) &:= \mathbf{\Xi}\,\alpha\texttt{-Entmax}(\beta\,\mathbf{\Xi}^\top \mathbf{x}), \label{eq:app_sparse_dyn}
\end{align}
where $\mathbf{\Psi}^\star_\alpha$ is the convex conjugate of the negative Tsallis entropy~\citep{tsallis} $\Psi_\alpha$, defined by
\begin{equation}
\Psi_{\alpha}(\mathbf{p}) :=
\begin{cases}
\tfrac{1}{\alpha(\alpha-1)}\sum_{i=1}^{M} (p_i^{\alpha} - p_i), & \alpha \neq 1,\\[2pt]
\sum_{i=1}^{M} p_i \log p_i, & \alpha = 1,
\end{cases}
\end{equation}
and the probability mapping $\alpha\texttt{-Entmax}(\mathbf{z}) := \argmax_{\mathbf{p}\in\Delta^{M-1}}\bigl[\langle\mathbf{p}, \mathbf{z}\rangle - \Psi_\alpha(\mathbf{p})\bigr]$. The mapping admits the closed form~\citep{entmaxx}
\begin{equation}
\alpha\texttt{-Entmax}(\mathbf{z}) = \bigl[(\alpha - 1)\mathbf{z} - \tau(\mathbf{z})\,\mathbf{1}\bigr]_+^{1/(\alpha-1)}, \label{eq:app_entmax}
\end{equation}
for a unique threshold function $\tau:\mathbb{R}^M\to\mathbb{R}$, so entries below $\tau(\mathbf{z})/(\alpha-1)$ map to zero, yielding sparsity. Throughout, $\kappa(\mathbf{z})$ denotes the number of nonzero entries. The parameter $\alpha$ controls sparsity: $\alpha=1$ recovers softmax (and hence $\mathcal{T}_1 \equiv \mathcal{T}_{\texttt{Dense}}$~\citep{entmaxx}); $\alpha = 2$ corresponds to $\texttt{sparsemax}$~\citep{sparsemax}, computable via sorting~\citep{held1974validation}, finite iterative projection~\citep{michelot1986finite}, or bisection~\citep{liu2009efficient}; and $\alpha = 1.5$ admits an exact sorting-based algorithm~\citep{entmaxx}. For general $1 < \alpha < 2$, a bisection scheme applies~\citep{entmaxx}.

\noindent\textbf{Extension to sparse attention.} Analogously to~\eqref{eq:app_at1}--\eqref{eq:app_at2}, the sparse retrieval dynamics extend directly to attention. For multiple queries,
\begin{equation}
\mathcal{T}_\alpha(\mathbf{X}) = \alpha\texttt{-Entmax}\bigl(\mathbf{Q}\mathbf{K}^\top/\sqrt{d}\bigr)\,\mathbf{V} := \texttt{At}_\alpha(\mathbf{Q}, \mathbf{K}, \mathbf{V}),
\end{equation}
and for cross-attention,
\[
\mathcal{T}_\alpha(\mathbf{X}, \mathbf{Y}) = \alpha\texttt{-Entmax}\bigl(\mathbf{X}\mathbf{W}_Q\mathbf{W}_K^\top\mathbf{Y}^\top/\sqrt{d}\bigr)\,\mathbf{Y}\mathbf{W}_V = \texttt{At}_\alpha(\mathbf{W}_Q\mathbf{X},\, \mathbf{W}_K\mathbf{Y},\, \mathbf{W}_V\mathbf{Y}).
\]
This extension lets the network selectively retrieve relevant patterns by zeroing out unrelated keys, while inheriting the theoretical properties of sparse associative memories.

For the retrieval-error statements that follow, we use the definition of~\citep{pladis}, adapted from~\citep{hopfield, sparsehop, stanhop}.

\begin{definition}[Pattern stored and retrieved~\citep{hopfield, sparsehop, stanhop, pladis}]
Suppose every pattern $\xib_\mu$ lies in a ball $B_\mu$. We say $\xib_\mu$ is \textbf{stored} if there exists a unique fixed point $\mathbf{x}^\star_\mu \in B_\mu$ to which all $\mathbf{x} \in B_\mu$ converge, with the $\{B_\mu\}$ disjoint. We say $\xib_\mu$ is \textbf{retrieved} with error $\epsilon$ if $\norm{\mathcal{T}(\mathbf{x}) - \xib_\mu} \le \epsilon$ for all $\mathbf{x} \in B_\mu$.
\end{definition}

Setting $m := \max_\nu \norm{\xib_\nu}$, the next proposition collects the retrieval-error bounds from the Hopfield literature that we use throughout the proofs.

\begin{proposition}[Retrieval Error~\citep{hopfield, sparsehop, stanhop, pladis}]\label{thm:retrieval_err}
Let $\mathcal{T}_\alpha$ be the Hopfield retrieval with $\alpha\texttt{-Entmax}$. Then, for $\alpha=1$,
\begin{equation}
\norm{\mathcal{T}_1(\mathbf{x}) - \xib_\mu} \le 2m(M-1)\exp\!\Bigl\{-\beta\bigl(\langle\xib_\mu, \mathbf{x}\rangle - \max_{\nu\ne\mu}\langle\xib_\nu, \mathbf{x}\rangle\bigr)\Bigr\}; \label{eq:app_ineq_hop}
\end{equation}
for $1 < \alpha \le 2$,
\begin{equation}
\norm{\mathcal{T}_\alpha(\mathbf{x}) - \xib_\mu} \le m + m\,\kappa\!\Bigl[(\alpha-1)\beta\,\delta_{\kappa}(\mathbf{x})\Bigr]^{\tfrac{1}{\alpha-1}}, \label{eq:app_ineq_new1}
\end{equation}
with $\delta_{\kappa}(\mathbf{x}) := \max_\nu\langle\xib_\nu, \mathbf{x}\rangle - [\mathbf{\Xi}^\top\mathbf{x}]_{(\kappa+1)}$; for $\alpha = 2$,
\begin{equation}
\norm{\mathcal{T}_2(\mathbf{x}) - \xib_\mu} \le m + d^{1/2}\,m\beta\!\Bigl[\kappa\bigl(\max_\nu\langle\xib_\nu, \mathbf{x}\rangle - [\mathbf{\Xi}^\top\mathbf{x}]_{(\kappa)}\bigr) + 1/\beta\Bigr]; \label{eq:app_ineq_sparse}
\end{equation}
and for $\alpha > \alpha'$,
\begin{equation}
\norm{\mathcal{T}_\alpha(\mathbf{x}) - \xib_\mu} \le \norm{\mathcal{T}_{\alpha'}(\mathbf{x}) - \xib_\mu} \quad\text{for } \mathbf{x}\in B_\mu. \label{eq:app_ineq_mono}
\end{equation}
The bound~\eqref{eq:app_ineq_hop} follows directly from the softmax, since $\norm{\mathcal{T}_1(\mathbf{x})-\xib_\mu}=\norm{\sum_{\nu\ne\mu}p_\nu(\xib_\nu-\xib_\mu)}\le 2m(1-p_\mu)$ and $1-p_\mu\le(M-1)e^{-\beta\Delta(\mathbf{x})}$; \citet{sparsehop} state it with the pattern overlap $\max_{\nu\ne\mu}\langle\xib_\mu,\xib_\nu\rangle$ in place of $\max_{\nu\ne\mu}\langle\xib_\nu,\mathbf{x}\rangle$. The bounds~\eqref{eq:app_ineq_new1},~\eqref{eq:app_ineq_sparse}, and~\eqref{eq:app_ineq_mono} are established, respectively, in~\citep{pladis},~\citep{sparsehop}, and~\citep{sparsehop, stanhop, pladis} (against the dense retrieval for $\alpha=2$ and for $\alpha\texttt{-Entmax}$; \citep{pladis} restates it for general $\alpha>\alpha'$). We adopt the convention $[\mathbf{\Xi}^\top\mathbf{x}]_{(M+1)} := [\mathbf{\Xi}^\top\mathbf{x}]_{(M)} - M^{1-\alpha}/(\alpha-1)$.
\end{proposition}
Theorem~\ref{thm:directional}(i) uses the pointwise ordering~\eqref{eq:app_ineq_mono} on the storage ball (Theorem~2.2 of~\citep{sparsehop}, Theorem~3.1 of~\citep{stanhop}), together with the exact retrieval above the collapse threshold.

\begin{corollary}[Noise Robustness; linear:~\citep{sparsehop, stanhop}, polynomial:~\citep{pladis}]\label{cor:noise_rob}
For noisy queries $\tilde{\mathbf{x}} = \mathbf{x} + \boldsymbol{\eta}$ or noisy memories $\tilde{\xib}_\mu = \xib_\mu + \boldsymbol{\eta}$, the noise impact on $\norm{\mathcal{T}_2(\mathbf{x}) - \xib_\mu}$ is \textbf{linear}; for $1 < \alpha \le 2$ it scales \textbf{polynomially} with order $\tfrac{1}{\alpha-1}$; whereas for $\alpha = 1$ it grows \textbf{exponentially} in the noise.
\end{corollary}

\noindent\textbf{Retrieval dynamics as fixed-point iteration.} The retrieval dynamics $\mt$ of MHNs~\citep{hopfield, sparsehop} are Picard iterations $\mbx_{k+1}=\mt(\mbx_k)$ toward a stationary point $\mbx_{\star}$ of the energy, near a stored memory when the patterns are well separated.

\begin{proposition}[Generalized Sparse Hopfield Retrieval Dynamics; Lemma~3.2 of~\citep{stanhop}]\label{thm:fp}
The retrieval dynamics of the generalized sparse Hopfield model is a monotonic one-step update:
\begin{equation}
\mbx_{k+1} = \mt(\mbx_k) := \mbXi\,\alpha\texttt{-Entmax}(\beta\mbXi^\top\mbx_k).
\end{equation}
\end{proposition}
Attention dynamics are thus the mechanism for finding fixed points $\mbx_{\star}$ in MHNs, and the standard attention layer performs exactly one iteration of the update rule: with $\beta=1/\sqrt{d}$, $\mt_{(\cdot)}(\mbx)\to\Att_{(\cdot)}(\Qb,\Kb,\Vb)$, where $\mbx$ is the query and $(\Kb,\Vb)$ are the stored memories. This fixed-point view lets us compare two retrieval operators that share the same memories but differ in the probability mapping.

\subsection{Proof of Theorem~\ref{thm:directional}}\label{app:directional_proof}
\noindent\textbf{Setup.} As in the main text, $\Tca(\mbx) = \Vb\,\mbp_\alpha(\mbx)$ and $\Tcd(\mbx) = \Vb\,\mbp_{\texttt{dense}}(\mbx)$ share $(\mathbf{Q},\mathbf{K},\mathbf{V})$ and conditioning $\mbc$, so the discrepancy $r(\mbx) := \Tca(\mbx) - \Tcd(\mbx) = \Vb\,\boldsymbol{\delta}(\mbx)$ lies in $\mathrm{span}(\Vb)$, with $\boldsymbol{\delta}(\mbx) := \mbp_\alpha(\mbx) - \mbp_{\texttt{dense}}(\mbx)\in\mathbb{R}^M$ a difference of two simplex points. We write $z_j := \langle\xib_j,\mbx\rangle$ for the logits, so that $\Delta(\mbx) = z_\mu - \max_{j\ne\mu} z_j$.

\noindent\textbf{(i) Convergence rate separation.} \emph{Step 1 (pointwise ordering).} By~\eqref{eq:app_ineq_mono} with $\alpha'=1$ (the dense case, since $1\texttt{-Entmax}\equiv\texttt{Softmax}$~\citep{entmaxx}) and $\alpha\in(1,2]$, $\norm{\Tca(\mbx)-\xib_\mu}\le\norm{\Tcd(\mbx)-\xib_\mu}$ for $\mbx$ in the storage ball~\citep{sparsehop, stanhop}, which is the first statement of (i).

\noindent\emph{Step 2 (exact retrieval above the collapse threshold).} $\alpha\texttt{-Entmax}$ has the closed form~\citep{entmaxx}
\[
(\mbp_\alpha)_i \;=\; \bigl[(\alpha-1)\beta z_i - \tau\bigr]_+^{1/(\alpha-1)},
\]
with $\tau$ chosen so that the entries sum to one. The support is $\{\mu\}$ if and only if $(\alpha-1)\beta z_\mu - \tau = 1$ and $(\alpha-1)\beta z_j - \tau \le 0$ for all $j\ne\mu$, i.e., $(\alpha-1)\beta(z_\mu - z_j)\ge 1$ for all $j\ne\mu$, which is $\beta\ge\beta_0(\mbx) = 1/((\alpha-1)\Delta(\mbx))$. In that case $\mbp_\alpha = \mathbf{e}_\mu$ and $\Tca(\mbx) = \xib_\mu$. Therefore, for $\beta\ge\beta_0(\mbx)$, $\norm{\Tca(\mbx)-\xib_\mu} = 0 \le \norm{\Tcd(\mbx)-\xib_\mu}$, which is~\eqref{eq:rate_sep}.

\noindent\textbf{(ii) Directional alignment.} \emph{Step 1 (identity).} For any $a,b,c$ in a Hilbert space,
\[
\langle a-b,\, c-b\rangle \;=\; \tfrac{1}{2}\bigl(\norm{a-b}^2 + \norm{c-b}^2 - \norm{a-c}^2\bigr).
\]
With $a=\Tca(\mbx)$, $b=\Tcd(\mbx)$, and $c=\xib_\mu$, we have $a-b = r(\mbx)$ and $c-b = \xib_\mu - \Tcd(\mbx)$, so
\[
\bigl\langle r(\mbx),\, \xib_\mu - \Tcd(\mbx)\bigr\rangle \;=\; \tfrac{1}{2}\Bigl(\norm{r(\mbx)}^2 + \norm{\xib_\mu - \Tcd(\mbx)}^2 - \norm{\xib_\mu - \Tca(\mbx)}^2\Bigr).
\]
\noindent\emph{Step 2 (inequality).} Dropping the non-negative term $\norm{\xib_\mu - \Tcd(\mbx)}^2$ yields
\[
\bigl\langle r(\mbx),\, \xib_\mu - \Tcd(\mbx)\bigr\rangle \;\ge\; \tfrac{1}{2}\Bigl(\norm{r(\mbx)}^2 - \norm{\Tca(\mbx) - \xib_\mu}^2\Bigr),
\]
whose right-hand side is non-negative whenever $\norm{r(\mbx)} \ge \norm{\Tca(\mbx) - \xib_\mu}$.\\
\noindent\emph{Step 3 (collapse regime).} For $\beta\ge\beta_0(\mbx)$, $\Tca(\mbx) = \xib_\mu$ by (i), so $r(\mbx) = \xib_\mu - \Tcd(\mbx)$ and the inner product equals $\norm{\xib_\mu - \Tcd(\mbx)}^2$.

Therefore, $r(\mbx)$ is non-negatively aligned with $\xib_\mu - \Tcd(\mbx)$ whenever $\norm{r(\mbx)} \ge \norm{\Tca(\mbx) - \xib_\mu}$, and exactly equals that residual above the collapse threshold, which is~\eqref{eq:dir_align}.\hfill$\square$

\noindent\textbf{Remark (coordinate-wise sharpening).} Coordinate-wise sharpening of $\mbp_\alpha$ relative to $\mbp_{\texttt{dense}}$ ($\delta_\mu\ge 0$ and $\delta_j\le 0$ for all $j\ne\mu$) does not hold in general: for logits $\beta\mathbf{z} = (3,\,2.7,\,1.8)$, $1.5\texttt{-Entmax}$ moves mass from the third token to both of the others, giving $\boldsymbol{\delta} = (+0.100,\,+0.019,\,-0.119)$. This is why (ii) is stated through retrieval errors rather than through the entries of $\boldsymbol{\delta}$.

\noindent\textbf{Quantitative form of the collapse threshold.} For $\alpha=1.5$ the condition $\beta\ge\beta_0(\mbx)$ reads $\beta\Delta(\mbx)\ge 2$, a scaled top-1/top-2 logit margin of at least $2$. Appendix~\ref{app:diag} reports the measured margins and support sizes, which lie below this threshold; in that regime the alignment in (ii) is a measured rather than a certified property.

\subsection{Proof of Corollary~\ref{cor:parallel}}\label{app:par_proof}
Throughout, $c(\mbx):=\langle\boldsymbol{\delta}(\mbx),\mbp_\alpha(\mbx)\rangle/\norm{\mbp_\alpha(\mbx)}^2$ is the per-query coefficient of the projection of $\boldsymbol{\delta}$ onto $\mbp_\alpha$. \textbf{Sign of $c(\mbx)$.} By definition, $c(\mbx)\norm{\mbp_\alpha}^2 = \norm{\mbp_\alpha}^2-\langle\mbp_{\texttt{dense}},\mbp_\alpha\rangle \ge \norm{\mbp_\alpha}^2-\norm{\mbp_{\texttt{dense}}}\norm{\mbp_\alpha} = \norm{\mbp_\alpha}\bigl(\norm{\mbp_\alpha}-\norm{\mbp_{\texttt{dense}}}\bigr)$ by Cauchy--Schwarz, which is non-negative under the stated norm condition. Above the collapse threshold $\mbp_\alpha=\mathbf{e}_\mu$, so $\norm{\mbp_\alpha}=1\ge\norm{\mbp_{\texttt{dense}}}$ and $c(\mbx)=1-(\mbp_{\texttt{dense}})_\mu\in[0,1)$.\\
\noindent\textbf{Alignment.} Above the threshold $\Tca(\mbx)=\xib_\mu$, hence
\[
\bigl\langle \Tca(\mbx),\,\xib_\mu-\Tcd(\mbx)\bigr\rangle = \norm{\xib_\mu}^2-\sum_j(\mbp_{\texttt{dense}})_j\langle\xib_\mu,\xib_j\rangle \ge \norm{\xib_\mu}^2-\sum_j(\mbp_{\texttt{dense}})_j\norm{\xib_\mu}^2 = 0,
\]
using $\langle\xib_\mu,\xib_j\rangle\le\norm{\xib_\mu}^2$ and $\sum_j(\mbp_{\texttt{dense}})_j=1$. Multiplying by $c(\mbx)\ge 0$ gives~\eqref{eq:par_align}.\\
\noindent\textbf{Layer-wise scalar.} With $\mathbf{P}_\alpha=[\mbp_\alpha(\mbx_q)]_q$ stacked over the queries of a layer, $\langle\boldsymbol{\Delta},\mathbf{P}_\alpha\rangle_F/\norm{\mathbf{P}_\alpha}_F^2 = \sum_q c(\mbx_q)\norm{\mbp_\alpha(\mbx_q)}^2\big/\sum_q\norm{\mbp_\alpha(\mbx_q)}^2$, a convex combination of the per-query scalars, so it is non-negative whenever every $c(\mbx_q)$ is.\hfill$\square$\\
\noindent\textbf{Remark.} Below the collapse threshold the norm condition $\norm{\mbp_\alpha}\ge\norm{\mbp_{\texttt{dense}}}$ is not implied by sparsity alone and is instead verified empirically: the probability-space cosine between $\boldsymbol{\delta}$ and $\mbp_\alpha$, whose sign is that of $c(\mbx)$, reported in Appendix~\ref{app:diag}, and it is positive at every measured layer and step.

\subsection{Prompt-Attention Mass and Proof of Proposition~\ref{prop:share}}\label{app:share_proof}\label{app:mass_schematic}
In cross-attention the keys of an image query are the text tokens alone, so each row sums to one over the text keys; the end-of-text and padding keys act as an attention sink, and we measure $m$ on the prompt tokens only, which hold between 0.77 and 0.80 of each row on SDXL (Appendix~\ref{app:settings}); in joint attention the keys are text and image tokens together, and $m(\mbx)$ is the part of the row that falls on the text keys (Figure~\ref{fig:mass_schematic}).

\begin{figure}[h!]\centering\includegraphics[width=476bp]{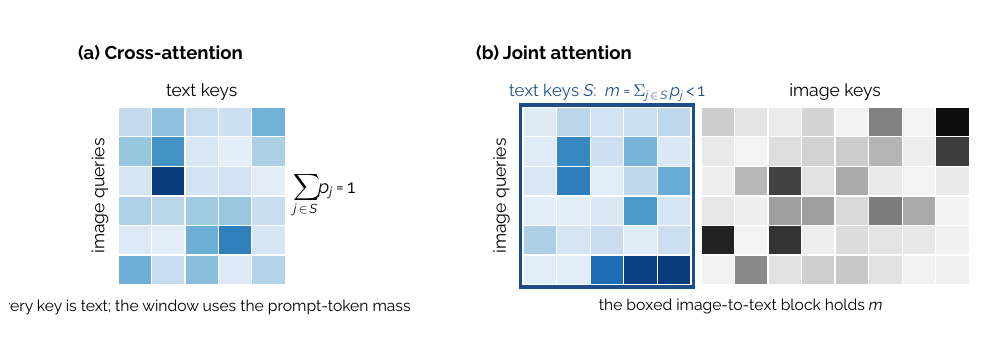}\caption{\textbf{Prompt-attention mass.} (a) Cross-attention: every key is text; the window uses the prompt-token mass. (b) Joint attention: $m$ is the mass of the image-to-text block.}\label{fig:mass_schematic}\end{figure}

\noindent\textbf{Proof of Proposition~\ref{prop:share}.} Fix an image query $\mbx$. Since $\boldsymbol{\Delta}^{\parallel}$ is the projection of $\boldsymbol{\Delta}$ onto $\mathbf{P}_\alpha$, its row is $\boldsymbol{\delta}(\mbx)=\bar c\,\mbp_\alpha(\mbx)$ with $\bar c:=\langle\boldsymbol{\Delta},\mathbf{P}_\alpha\rangle_F/\norm{\mathbf{P}_\alpha}_F^2$, so $\sum_j\delta_j(\mbx)=\bar c$. With $\bar{\mbv}_S(\mbx):=\sum_{j\in S}(\mbp_\alpha(\mbx))_j\mbv_j/m_\alpha(\mbx)$ and $\bar{\mbv}_I(\mbx):=\sum_{j\notin S}(\mbp_\alpha(\mbx))_j\mbv_j/(1-m_\alpha(\mbx))$, splitting the sum over keys into text and image keys gives $r^{\parallel}(\mbx)=\bar c\sum_j(\mbp_\alpha(\mbx))_j\mbv_j=\bar c\,\bigl(m_\alpha(\mbx)\bar{\mbv}_S(\mbx)+(1-m_\alpha(\mbx))\bar{\mbv}_I(\mbx)\bigr)$, which is~\eqref{eq:share}. In cross-attention every key is a text key, so $m_\alpha\equiv 1$ for the full text block and the image term is absent; the window of Section~\ref{sec:gate} uses the prompt-token part of this mass.\hfill$\square$

\subsection{Directional Consistency}\label{app:cor_proof}
Theorem~\ref{thm:directional} yields a structural corollary on the geometric soundness of extrapolating along $r(\mbx)$.

\begin{corollary}[Directional Consistency]\label{cor:dirconsist}
Let $\Tc'(\mbx) := \Tca(\mbx) + \lambda\,r(\mbx)$ with $\lambda\ge 0$. Under the setup of Theorem~\ref{thm:directional}:
\begin{enumerate}
\item[(a)] \textbf{Direction preservation.} $r(\mbx)\in\mathrm{span}(\Vb)$ and, under the condition of Theorem~\ref{thm:directional}(ii), is non-negatively aligned with $\xib_\mu-\Tcd(\mbx)$, so $\Tc'(\mbx)\in\mathrm{span}(\Vb)$ with no off-subspace drift.
\item[(b)] \textbf{Attractor up to the dense retrieval error.} Above the collapse threshold of Theorem~\ref{thm:directional}(i), $\Tca(\mbx)=\xib_\mu$ and $r(\mbx)=\xib_\mu-\Tcd(\mbx)$, so
\begin{equation}
\norm{\Tc'(\mbx)-\xib_\mu} \;=\; \lambda\,\norm{\Tcd(\mbx)-\xib_\mu} \;\le\; 2\lambda m(M-1)\,e^{-\beta\Delta(\mbx)}, \label{eq:attractor_err}
\end{equation}
with $m$ and $M$ as in Proposition~\ref{thm:retrieval_err}. Hence $\xib_\mu$ is a fixed point of $\Tc'$ in the limit $\beta\to\infty$, and at finite $\beta$ the extrapolation overshoots $\xib_\mu$ by $\lambda$ times the dense retrieval error, a quantity exponentially small in the scaled margin $\beta\Delta(\mbx)$.
\end{enumerate}
\end{corollary}

For the update of ReGA, $\Tca(\mbx)+\lambda s\,r^{\parallel}(\mbx)$, the same argument gives a deviation from $\xib_\mu$ of $\lambda s\norm{r^{\parallel}(\mbx)}$ above the collapse threshold, and the norm clamp keeps $s\norm{\boldsymbol{\Delta}^{\parallel}}_F\le\eta$, so the overshoot is at most $\lambda\eta\norm{\Vb}_2$ at every $\beta$. Along real sampling trajectories, the orthogonal component stays the smaller one (Appendix~\ref{app:diag}; Figure~\ref{fig:resmag}).

\begin{remark}[Separate key and value projections]\label{rem:kv}
An attention layer scores with keys $\Kb$ and reads out with values $\Vb$, whereas the Hopfield model uses one pattern matrix for both. Theorem~\ref{thm:directional} and Corollary~\ref{cor:dirconsist} use the keys only through $\mbp_\alpha,\mbp_{\texttt{dense}}$ and the values only through $\Vb$, so they hold verbatim with $\xib_\mu$ replaced by $\mbv_\mu$, the value of the dominant key, $\Delta(\mbx)$ by the key logit margin, and $m$ by $\max_j\norm{\mbv_j}$ (proof below); the Hopfield layer of \citet{hopfield} projects the values from the keys by an additional $\mathbf{W}_V$.
\end{remark}

\noindent\textbf{Proof of (a).} Both $\Tca(\mbx)$ and $\Tcd(\mbx)$ live in $\mathrm{span}(\Vb)$, so $r(\mbx)\in\mathrm{span}(\Vb)$. Since $\Tc'(\mbx) = \Tca(\mbx) + \lambda r(\mbx)$ with $\Tca(\mbx)\in\mathrm{span}(\Vb)$, we have $\Tc'(\mbx)\in\mathrm{span}(\Vb)$. Under the condition of Theorem~\ref{thm:directional}(ii), $r(\mbx)$ is non-negatively aligned with $\xib_\mu - \Tcd(\mbx)$, so the extrapolation does not introduce off-direction drift away from $\xib_\mu$.\\
\noindent\textbf{Proof of (b).} Above the collapse threshold, $\mbp_\alpha(\mbx)=\mathbf{e}_\mu$ (Step 2 of Appendix~\ref{app:directional_proof}), so $\Tca(\mbx)=\xib_\mu$ and $r(\mbx)=\xib_\mu-\Tcd(\mbx)$. Substituting in $\Tc'$,
\[
\Tc'(\mbx)-\xib_\mu \;=\; \bigl(\Tca(\mbx)-\xib_\mu\bigr) + \lambda\,r(\mbx) \;=\; \lambda\,\bigl(\xib_\mu-\Tcd(\mbx)\bigr),
\]
and taking norms gives the equality in~\eqref{eq:attractor_err}. The inequality is the dense bound~\eqref{eq:app_ineq_hop}, derived after Proposition~\ref{thm:retrieval_err}, whose exponent is $\beta\Delta(\mbx)$. The right-hand side vanishes as $\beta\to\infty$ but not at finite $\beta$, since $\Tcd(\mbx)$ is a strictly convex combination of all stored patterns whenever every logit is finite.\hfill$\square$\\
\noindent\textbf{Two-pattern example.} Let $d=1$, $\mathbf{\Xi}=[1,-1]$, $\alpha=3/2$, and $\beta>1$. For $x$ near $1$ the margin is $\Delta(x)=2x$, the collapse condition $(\alpha-1)\beta\Delta(x)\ge 1$ holds, and $\Tca(x)=1$, whereas $\Tcd(x)=\tanh(\beta x)$. Hence $r(1)=1-\tanh\beta>0$ and $\Tc'(1)=1+\lambda(1-\tanh\beta)\ne 1$ for every finite $\beta$ and $\lambda>0$: the stored pattern is not an exact fixed point of $\Tc'$. The deviation $\lambda(1-\tanh\beta)=2\lambda/(e^{2\beta}+1)$ is exactly $\lambda$ times the dense retrieval error, in agreement with~\eqref{eq:attractor_err} (here $m=1$, $M=2$, $\Delta(1)=2$, and the bound reads $2\lambda e^{-2\beta}$).\\
\noindent\textbf{Separate key and value projections (Remark~\ref{rem:kv}).} Let the layer score with keys $\{\mathbf{k}_j\}$ and read out with values $\{\mbv_j\}$, so that $\Tca(\mbx)=\sum_j (\mbp_\alpha)_j\mbv_j$ and $\Tcd(\mbx)=\sum_j(\mbp_{\texttt{dense}})_j\mbv_j$ with both probability vectors computed from the key logits $z_j=\langle\mathbf{k}_j,\mbx\rangle$. Theorem~\ref{thm:directional}(i) Step 2 concerns $\mbp_\alpha$ only and gives $\Tca(\mbx)=\mbv_\mu$ above the key-margin threshold; the polarization identity in (ii) is algebraic in $a=\Tca(\mbx)$, $b=\Tcd(\mbx)$, $c=\mbv_\mu$; and the dense error satisfies
\[
\norm{\Tcd(\mbx)-\mbv_\mu} = \Bigl\|\sum_{j\ne\mu}(\mbp_{\texttt{dense}})_j(\mbv_j-\mbv_\mu)\Bigr\| \le 2\max_j\norm{\mbv_j}\,\bigl(1-(\mbp_{\texttt{dense}})_\mu\bigr) \le 2\max_j\norm{\mbv_j}\,(M-1)\,e^{-\beta\Delta(\mbx)},
\]
which is~\eqref{eq:app_ineq_hop} with $m=\max_j\norm{\mbv_j}$ and the margin taken over key logits. All statements therefore transfer with $\xib_\mu\mapsto\mbv_\mu$.

\pagebreak\subsection{Stability via Orthogonal Weak Contraction}\label{sec:stable}\label{sec:proof}
We write $\Tcal(\mbx) := \Tca(\mbx) + \lambda s\bigl(r^{\parallel}(\mbx) + \zeta\, r^{\perp}(\mbx)\bigr)$, where $\zeta\in[0,1]$ is the retained fraction of the orthogonal component; ReGA uses $\zeta=0$. Let $\mbx^*$ be an attractor of $\Tca$ and $\Pi_\perp$ the orthogonal projector onto $(\mathrm{span}\{\Tca(\mbx^*)\})^{\perp}$, fixed along the iteration.

\begin{assumption}[Orthogonal Weak Contraction~\citep{rhoades2001some, alber1997principle}]\label{ass:weak_orth}
There exists a strictly increasing forcing function $\phi:\mathbb{R}_{\ge 0}\to\mathbb{R}_{\ge 0}$ with $\phi(0)=0$ such that, for $\Tcal$ with $\zeta=0$,
\begin{equation}
\norm{\Pi_\perp(\Tcal(\mbx)-\mbx^*)} \;\le\; \norm{\Pi_\perp(\mbx-\mbx^*)} - \phi\bigl(\norm{\Pi_\perp(\mbx-\mbx^*)}\bigr).
\end{equation}
\end{assumption}

\begin{proposition}[Asymptotic Convergence of the Orthogonal Error]\label{thm:conv_orth}
Suppose Assumption~\ref{ass:weak_orth} holds, and let $\{\mbx_t\}_{t\in\mathbb{N}}$ be generated by the iterated dynamics $\mbx_{t+1}=\Tcal(\mbx_t)$ with $\zeta=0$. Then the orthogonal error magnitude $u_t := \norm{\Pi_\perp(\mbx_t-\mbx^*)}$ satisfies $\lim_{t\to\infty}u_t = 0$.
\end{proposition}

\begin{proof}[\upshape\bfseries Proof]
Since $\Pi_{\perp}$ is fixed, applying Assumption~\ref{ass:weak_orth} at $\mbx = \mbx_t$ gives
\begin{equation}
u_{t+1} \;\le\; u_t - \phi(u_t), \qquad \forall t\ge 0, \label{eq:rec_app}
\end{equation}
with $\phi$ strictly increasing and $\phi(0)=0$. Hence $u_{t+1} < u_t$ whenever $u_t>0$, so $\{u_t\}$ is monotone non-increasing and bounded below by $0$. By the Monotone Convergence Theorem, $u_\infty := \lim_t u_t$ exists. Taking limits in~\eqref{eq:rec_app} gives $\phi(u_\infty)\le 0$, which forces $u_\infty=0$.
\end{proof}

\begin{remark}[Stability and robustness]\label{rem:stability}
Proposition~\ref{thm:conv_orth} concerns the iterated Hopfield dynamics, whereas an attention layer performs a single step (Proposition~\ref{thm:fp}), so we read it as a property of the update direction: with $\zeta=0$ the update adds no orthogonal component and Assumption~\ref{ass:weak_orth} can drive the orthogonal error down, whereas $\zeta=1$ re-injects it at every step. Appendix~\ref{app:diag} tests the observable consequence, that $r^{\perp}$ is harmful through its direction rather than its scale (Figure~\ref{fig:residual}).
\end{remark}


\section{Implementation and Settings}\label{app:impl}

\subsection{Sampling Protocol}\label{sec:detail}
Default schedulers and step counts are used throughout: 50 steps for SDXL, 4 for its step-distilled variants (Hyper-SDXL, DMD2, SDXL-Lightning), 28 for FLUX.1-Dev, and 4 for FLUX.1-Schnell, FLUX.2, and Qwen-Image, on a single NVIDIA H100 with batch size 1. For the SDXL evaluations in Table~\ref{tab_main}, we sample $10{,}000$ prompts from the validation set of MS-COCO. GenEval results use the official benchmark and report the mean overall score. For CFG-based comparisons we use scale $w=5.0$; for PAG and APG we use the recommended scales $3.0$ and $15.0$, respectively. For the step-distilled models and the FLUX.1 variants we follow the official baseline configurations. For ReGA, we set the clamp threshold $\eta=15$ and use $\zeta=0$ throughout. The user study protocol is in Appendix~\ref{app:userstudy}.

\noindent\textbf{MM-DiT backbones.} FLUX and Qwen-Image use joint attention over the concatenated text and image tokens. ReGA is applied to the image-to-text block of the joint attention map only (image queries attending to text keys), where the shared-memory setting of Theorem~\ref{thm:directional} holds; the block is then rescaled so that each query keeps the mass that the dense softmax assigns to the text tokens, following the official PLADIS implementation~\citep{pladis}. The remaining blocks are left as they are.

\FloatBarrier
\subsection{Per-Backbone Settings and Guided Windows}\label{app:settings}
Table~\ref{tab:settings} lists the settings of every backbone. $\lambda{=}10$ for the SDXL family comes from the sweep in Figure~\ref{fig:row}, and $\lambda{=}5$ for the MM-DiT backbones from the sweep on FLUX.2-4B in Table~\ref{tab:sweep4b}; FLUX.1-Dev uses $\lambda{=}12$ in the gate study of Appendix~\ref{app:gate}, and Qwen-Image-Lightning uses $\lambda{=}3$. $\alpha{=}1.5$ follows PLADIS, $\zeta{=}0$ follows Theorem~\ref{thm:directional}, and $\eta{=}15$ is used throughout: results are insensitive to $\eta$ over $[1,15]$ on SDXL (Table~\ref{tab:eta_sweep}) and over $[5,20]$ on FLUX.2-4B (Table~\ref{tab:sweep4b}), and degrade without the clamp (Table~\ref{tab:ablation}; Table~\ref{tab:clamp4b}). PLADIS uses its recommended scale ($\lambda{=}2.0$) with $\alpha{=}1.5$ and no clamp. Unlike the heuristic layer selection of PAG and SEG, ReGA is applied to all cross-attention layers on SDXL, following PLADIS.

\noindent\textbf{Measuring $m_t$ for the window.} As defined in Section~\ref{sec:gate}, $m_t$ is the dense image-to-text attention mass of the base sampler at step $t$, averaged over the guided layers, heads, image queries, and prompts. The window $k$ of Eq.~\eqref{eq:window} is the first step at which the cumulative mass reaches half of its total over the $T$ steps, and it is read once per backbone. This gives steps 1--2 of 4 on FLUX.2-4B, FLUX.2-9B, FLUX.1-Schnell, and Qwen-Image, and 1--13 of 28 on FLUX.1-Dev, and 1--26 of 50 on SDXL. \rnow{On SDXL the prompt-token mass stays between 0.77 and 0.80 over the 50 steps, so the cumulative half is reached at step 26; a prompt switch after step 25 leaves the image about 90\% committed to the original prompt. On SDXL the window is therefore not located by a decay of the mass but by the cumulative rule alone, and the prompt-switch test confirms that it ends where the prompt is committed. The 4-step distilled SDXL models are guided at every step.} Figure~\ref{fig:cummass} shows how the mass is spread over the steps on each backbone and the resulting window.
\begin{figure}[h!]\centering\includegraphics[width=458bp]{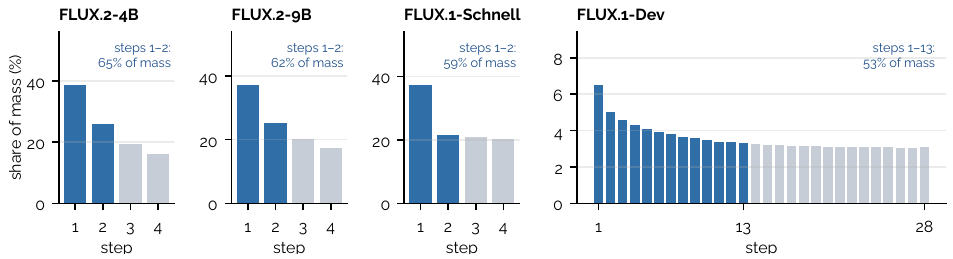}\caption{Share of the prompt-attention mass $m_t$ of the base sampler at each step. The window of Eq.~\eqref{eq:window} (blue) is the shortest run of initial steps that holds half of the total.}\label{fig:cummass}\end{figure}

\begin{table}[h!]
\centering\small
\caption{Per-backbone settings.}
\label{tab:settings}
\begin{tabular}{lccccl}
\toprule
Backbone & $\lambda$ & $\eta$ & $\alpha$ & Window & Layers \\
\midrule
SDXL & 10 & 15 & 1.5 & \rnow{1--26 of 50} & all cross-attention \\
SDXL distilled (4 steps) & 10 & 15 & 1.5 & all steps & all cross-attention \\
FLUX.1-Schnell & 5 & 15 & 1.5 & 1--2 of 4 & image-to-text, double-stream \\
FLUX.1-Dev & 5; 12 (gate study) & 15 & 1.5 & 1--13 of 28 & image-to-text, double-stream \\
FLUX.2 (4B, 9B) & 5 & 15 & 1.5 & 1--2 of 4 & image-to-text, double-stream \\
Qwen-Image & 5; 3 (Lightning) & 15 & 1.5 & 1--2 of 4 & image-to-text \\
\bottomrule
\end{tabular}
\end{table}

\FloatBarrier
\subsection{Implementation and Cost}\label{app:overhead}
Table~\ref{tab:impl_compare} gives the wall-clock time and peak memory behind Figure~\ref{fig:gate}a. All numbers are end-to-end per image (text encoding, denoising, VAE decode) on one NVIDIA H100 80GB, $1024\times1024$, batch 1, with two warm-up and five timed runs (seven for SDXL); SDXL runs with CFG (scale 5.0, batch of 2), and the FLUX models are guidance-distilled and run without CFG. Peak memory is the maximum allocated over the run, weights included. All rows use our accelerated implementation, which is the released one.

\begin{table}[h!]
\centering\small
\caption{\rnow{Time per image (s) on one H100 and peak memory (GiB). ``Every step'' is ReGA without the gate; the guided window is given for ReGA.}}
\label{tab:impl_compare}
\rnow{\begin{tabular}{lcccccc}
\toprule
Model (steps) & Base & PLADIS & Every step & ReGA & Window & Peak memory \\
\midrule
SDXL+CFG (50) & 3.497 & 4.852 & 6.029 & 4.576 & 1--26 & 8.98 \\
FLUX.1-Schnell (4) & 0.883 & 0.997 & 1.005 & 0.931 & 1--2 & 33.83 \\
FLUX.1-Dev (28) & 5.296 & 5.998 & 6.052 & 5.554 & 1--13 & 33.85 \\
FLUX.2-4B (4) & 0.558 & 0.602 & 0.613 & 0.581 & 1--2 & 17.35 \\
FLUX.2-9B (4) & 0.850 & 0.910 & 0.909 & 0.870 & 1--2 & 34.76 \\
\bottomrule
\end{tabular}}
\end{table}

\enlargethispage{\baselineskip}\noindent\textbf{What the accelerated implementation changes.} (i) $\alpha\texttt{-Entmax}$ is computed with the root-finding Triton kernel of AdaSplash~\citep{adasplash} instead of the sort-based reference (two kernels instead of 28). (ii) The $\zeta{=}0$ projection is rank one, so $r^{\parallel}$ is computed as a per-sample scalar rescaling of $\Tca(\mbx)$ rather than an explicit projection with fp64 temporaries (agreement with the explicit form $2.2\times10^{-16}$ in fp64). (iii) On the FLUX family the guided weights are linear in the sparse and dense maps, so the image-to-text block is recomposed from flash attention~\citep{flashattention} and its log-sum-exp without materializing the $4608\times4608$ joint attention matrix and its fp32 copy (4.5~GiB; agreement $8.8\times10^{-16}$). (iv) SDXL is launch-bound rather than GPU-bound (48\% GPU-busy under standard execution), so the all-zero bias and unit-temperature operations are skipped and the UNet forward is replayed from a CUDA graph; this is applied to the base sampler, PLADIS, and ReGA alike, and latents are bit-identical to standard execution. The accelerated path is algebraically identical to the reference except for $\mathbf{Q}\mathbf{K}^{\top}$ in fp16 on SDXL and two entmax iterations. With the reference implementations, SDXL takes 5.23~s (PLADIS) and 6.21~s (ReGA) per image, and peak memory rises to 36.44~GiB on FLUX.1-Schnell and 19.96~GiB on FLUX.2-4B; the accelerated implementation keeps memory at the base level. On SDXL the sparse attention accounts for about two thirds of the added cost and the projection and clamp for the rest; on FLUX the projection and clamp add no measurable cost beyond the sparse attention (within $\pm0.02$~s).

\section{Additional Results}\label{app:results}

\subsection{Gate}\label{app:gate}
Table~\ref{tab:wogate_full} gives the full results of Tables~\ref{tab_main} and~\ref{tab_main_flux} with and without the gate, and Table~\ref{tab:gate_app} collects the gate study behind Figure~\ref{fig:gate}a,b (GenEval, 553 prompts). The early-versus-late comparison of Figure~\ref{fig:gate}b uses the windows 1--2 on FLUX.2-4B, 1--3 on FLUX.2-9B, and 1--14 on FLUX.1-Dev, and their complements (3--4, 4, and 15--28) as late-only guidance\rnow{; it was run before the window rule was fixed, so its windows on FLUX.2-9B and FLUX.1-Dev are one step longer}. The online gate turns guidance on at step $t$ and block $b$ when $m_{t,b}/m_{1,b}\ge 1/2$ on the actual sampling trajectory, where $m_{t,b}$ is the prompt-attention mass of block $b$ at step $t$; it matches the fixed window while guiding more block-steps. FLUX.1-Dev rows use $\lambda{=}12$ and one image per prompt; the other rows use the default $\lambda$ of Table~\ref{tab:settings}.
\begin{table}[h!]
\centering\small
\caption{Full results with and without the gate. ``w/o gate'' guides every step; ReGA guides the window of Table~\ref{tab:settings}. The 4-step distilled SDXL models are guided at every step, so their rows in Table~\ref{tab_main} are unchanged.}
\label{tab:wogate_full}
\begin{tabular*}{0.82\linewidth}{@{\extracolsep{\fill}}llccccccc}
\toprule
Backbone & Variant & Window & GE~$\uparrow$ & FID~$\downarrow$ & CS~$\uparrow$ & IR~$\uparrow$ & PS~$\uparrow$ & HPSv2~$\uparrow$ \\
\midrule
SDXL+CFG & w/o gate & all 50 & 0.607 & 14.63 & 27.26 & 0.812 & 22.40 & 0.271 \\
 & ReGA & 1--26 & 0.619 & 14.65 & 27.44 & 0.822 & 22.39 & 0.271 \\
SDXL+APG & w/o gate & all 50 & 0.613 & 13.22 & 27.55 & 0.851 & 22.24 & 0.274 \\
 & ReGA & 1--26 & 0.615 & 13.24 & 27.54 & 0.848 & 22.20 & 0.271 \\
SDXL+CFG+PAG & w/o gate & all 50 & 0.619 & 14.32 & 26.11 & 0.822 & 22.36 & 0.281 \\
 & ReGA & 1--26 & 0.617 & 14.40 & 26.23 & 0.824 & 22.36 & 0.283 \\
\bottomrule
\end{tabular*}

\vspace{0.6em}
\begin{tabular*}{0.82\linewidth}{@{\extracolsep{\fill}}llcccccc}
\toprule
Backbone & Variant & Window & GE~$\uparrow$ & CS~$\uparrow$ & IR~$\uparrow$ & PS~$\uparrow$ & HPSv2~$\uparrow$ \\
\midrule
FLUX.1-Schnell & w/o gate & all & 0.738 & 26.56 & 1.114 & 22.80 & 0.301 \\
 & ReGA & 1--2 & 0.735 & 26.66 & 1.135 & 22.86 & 0.299 \\
FLUX.1-Dev & w/o gate & all & 0.737 & 25.76 & 1.127 & 23.14 & 0.309 \\
 & ReGA & 1--13 & 0.731 & 25.96 & 1.156 & 23.08 & 0.307 \\
Qwen-Image & w/o gate & all & 0.842 & 27.79 & 1.171 & 23.19 & 0.312 \\
 & ReGA & 1--2 & 0.841 & 28.01 & 1.167 & 23.17 & 0.311 \\
FLUX.2-4B & w/o gate & all & 0.843 & 27.73 & 1.173 & 23.10 & 0.313 \\
 & ReGA & 1--2 & 0.843 & 27.71 & 1.181 & 23.11 & 0.313 \\
FLUX.2-9B & w/o gate & all & 0.878 & 27.84 & 1.235 & 23.21 & 0.316 \\
 & ReGA & 1--2 & 0.877 & 28.00 & 1.251 & 23.22 & 0.317 \\
\bottomrule
\end{tabular*}
\end{table}

\begin{table}[h!]
\centering\small
\caption{Gate study (GenEval, 553 prompts). $^{\ast}$: paired 95\% bootstrap CI excludes zero.}
\label{tab:gate_app}
\begin{tabular}{lll}
\toprule
Setting & Comparison & Value \\
\midrule
Late only & vs.\ base: 4B (3--4) / 9B (4) / Dev (15--28) & $-$.002 / $-$.002 / +.001 \\
Online gate & 4B / 9B (fixed window in parentheses) & .840 (.841) / .868 (.867) \\
Dev & base / ReGA (1--13) & .678 / .712$^{\ast}$ \\
Dev & window 1--6 / 1--3, vs.\ 1--14 & $-$.007 / $-$.023$^{\ast}$ \\
Dev & 7 / 10 / 14 / 21 / 28 guided steps & .691 / .706 / .701 / .709 / .708 \\
\bottomrule
\end{tabular}
\end{table}

\noindent\textbf{Predicted images along sampling.} Figures~\ref{fig:x0_app} and~\ref{fig:x0_4b} extend Figure~\ref{fig:x0} to more prompts on which the base sampler fails, on FLUX.1-Dev (window 1--13 of 28) and FLUX.2-4B (window 1--2 of 4). In every case the layout is fixed within the window: ReGA and the ungated update agree from the first guided steps on, while guiding only after the window leaves the trajectory of the base sampler unchanged.
\begin{figure}[p]\centering\includegraphics[width=480bp]{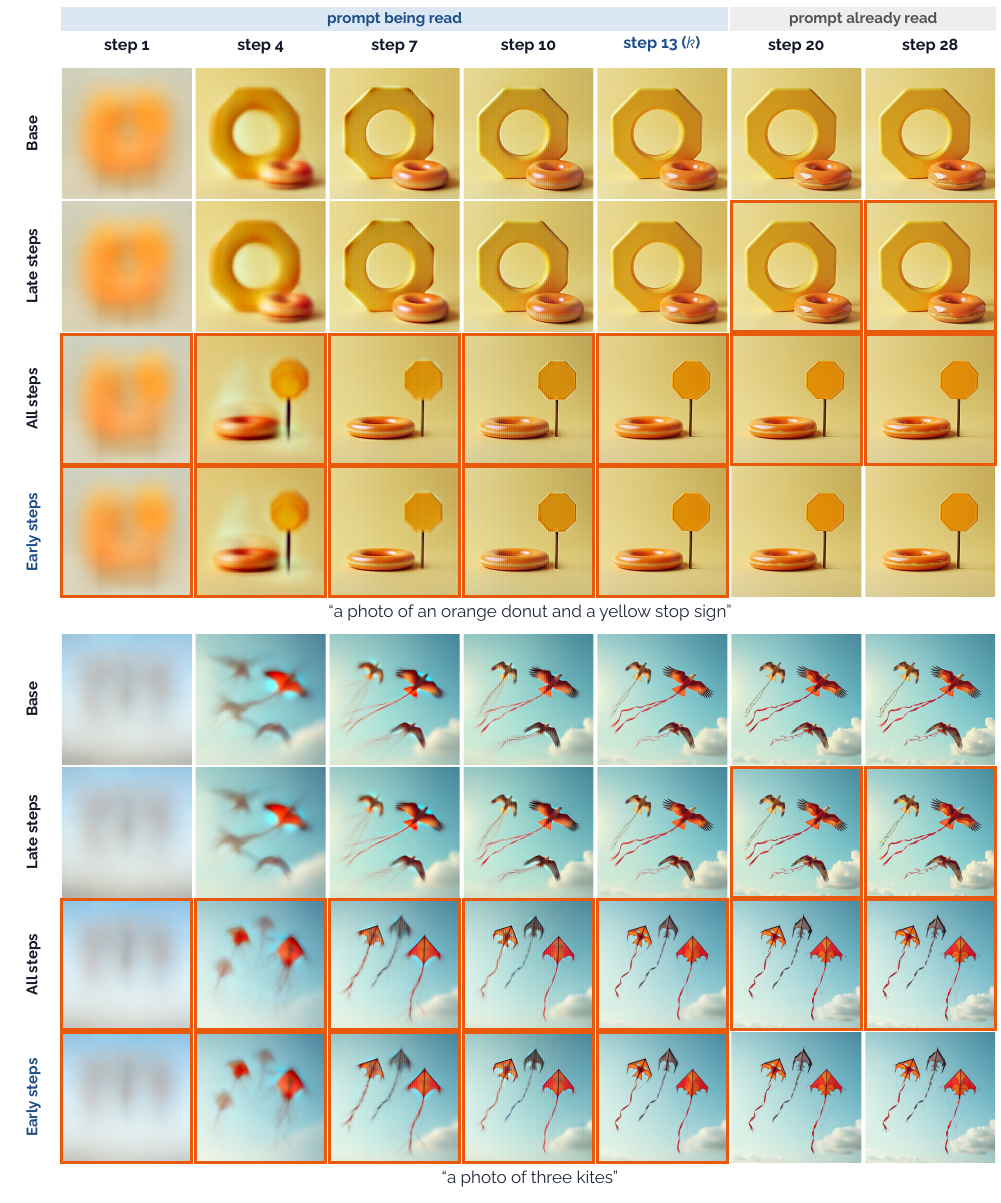}\caption{\rnow{\textbf{Predicted $\hat{\mbx}_0$ along sampling on FLUX.1-Dev.} Orange frames mark the guided steps; the prompt is read up to step $k$.}}\label{fig:x0_app}\end{figure}

\subsection{\rnow{GenEval per Category}}\label{app:cat}
\rnow{Table~\ref{tab:cat_main} gives GenEval per category on the MM-DiT backbones, referenced in Section~\ref{sec:analysis}.}
\begin{table}[h!]
\centering
\caption{\newred{\textbf{GenEval per category on MM-DiT backbones} ($\times 100$). Each pair shares code and prompts.}}
\label{tab:cat_main}
\footnotesize
\setlength{\tabcolsep}{2.2pt}
\begin{tabular}{llcccccc}
\toprule
Backbone & & Single & Two & Count & Colors & Pos. & Attr. \\
\midrule
\multirow{2}{*}{FLUX.1-Schnell} & Base & 99.4 & 86.4 & 62.8 & 71.5 & 26.0 & 44.3 \\
 & \cellcolor{green!10}ReGA & \cellcolor{green!10}99.4 & \cellcolor{green!10}\bf 91.9 & \cellcolor{green!10}\bf 78.8 & \cellcolor{green!10}\bf 80.0 & \cellcolor{green!10}\bf 35.0 & \cellcolor{green!10}\bf 55.6 \\
\midrule
\multirow{2}{*}{FLUX.1-Dev} & Base & 98.8 & 84.8 & 76.2 & 74.5 & 23.0 & 45.0 \\
 & \cellcolor{green!10}ReGA & \cellcolor{green!10}98.8 & \cellcolor{green!10}\bf 88.9 & \cellcolor{green!10}\bf 80.0 & \cellcolor{green!10}\bf 85.0 & \cellcolor{green!10}\bf 33.1 & \cellcolor{green!10}\bf 53.0 \\
\midrule
\multirow{2}{*}{FLUX.2-4B} & Base & 99.7 & 93.7 & 84.4 & \bf 87.2 & 66.7 & 63.0 \\
 & \cellcolor{green!10}ReGA & \cellcolor{green!10}99.7 & \cellcolor{green!10}\bf 96.2 & \cellcolor{green!10}\bf 85.8 & \cellcolor{green!10}86.7 & \cellcolor{green!10}\bf 73.8 & \cellcolor{green!10}\bf 63.7 \\
\midrule
\multirow{2}{*}{FLUX.2-9B} & Base & \bf 99.7 & 96.0 & 79.1 & 91.0 & 71.5 & \bf 72.8 \\
 & \cellcolor{green!10}ReGA & \cellcolor{green!10}99.4 & \cellcolor{green!10}\bf 97.5 & \cellcolor{green!10}\bf 87.7 & \cellcolor{green!10}91.0 & \cellcolor{green!10}\bf 76.0 & \cellcolor{green!10}72.0 \\
\bottomrule
\end{tabular}

\end{table}

\subsection{Hyperparameter Sensitivity}\label{app:hparam}
On FLUX.2-4B with the gate (base GenEval .824), GenEval stays within a non-significant margin of the default over $\lambda\in[1,10]$ and $\eta\in[5,20]$ (Table~\ref{tab:sweep4b}), and on COCO-1k, $\lambda$ above the default lowers HPSv2 significantly. On SDXL+CFG the results are insensitive to $\eta$ over $[1,15]$ (Table~\ref{tab:eta_sweep}), and on SDXL+APG CLIPScore and ImageReward rise with $\lambda$ up to the default $\lambda{=}10$ and stay stable beyond it (Figure~\ref{fig:row}; Table~\ref{tab:lambda_full}).
\begin{table}[t]
\centering\small
\begin{minipage}[t]{0.48\linewidth}\centering
\caption{GenEval sweeps on FLUX.2-4B (ReGA, steps 1--2; 553 prompts). Bold: default.}\label{tab:sweep4b}
\begin{tabular}{lccccc}\toprule $\lambda$ ($\eta{=}15$) & 1 & 3 & \textbf{5} & 7 & 10 \\ GE & .835 & .834 & \textbf{.841} & .840 & .842 \\ \midrule $\eta$ ($\lambda{=}5$) & 5 & 10 & \textbf{15} & 20 & \\ GE & .836 & .837 & \textbf{.841} & .838 & \\ \bottomrule\end{tabular}
\end{minipage}\hfill
\begin{minipage}[t]{0.48\linewidth}\centering
\caption{$\eta$ sweep (SDXL+CFG, $\lambda{=}10$).}\label{tab:eta_sweep}
\begin{tabular}{lccc}
\toprule
$\eta$ & GE~$\uparrow$ & CS~$\uparrow$ & IR~$\uparrow$ \\
\midrule
1 & 0.598 & 27.11 & 0.805 \\
5 & 0.602 & 27.22 & 0.808 \\
10 & 0.605 & 27.25 & 0.811 \\
\rowcolor{green!10} 15 (default) & 0.607 & 27.26 & 0.812 \\
\bottomrule
\end{tabular}
\end{minipage}

\vspace{0.8em}
\begin{minipage}[c]{0.36\linewidth}\centering
\includegraphics[width=\linewidth]{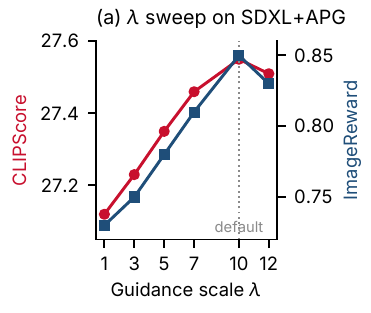}
\addtocounter{figure}{1}\captionof{figure}{$\lambda$ sweep on SDXL+APG (numbers in Table~\ref{tab:lambda_full}). The NAG comparison is in Table~\ref{tab:nag_full}.}
\label{fig:row}
\end{minipage}\hfill
\begin{minipage}[c]{0.58\linewidth}\centering\small
\captionof{table}{$\lambda$ sweep on SDXL+APG ($5$K samples).}
\label{tab:lambda_full}
\begin{tabular}{lcccccc}
\toprule
$\lambda$ & 1 & 3 & 5 & 7 & 10 & 12 \\
\midrule
CS~$\uparrow$ & 27.12 & 27.23 & 27.35 & 27.46 & \bf 27.55 & 27.51 \\
IR~$\uparrow$ & 0.73 & 0.75 & 0.78 & 0.81 & \bf 0.85 & 0.83 \\
\bottomrule
\end{tabular}
\end{minipage}
\end{table}

\subsection{Component Ablations}\label{app:ablations}\label{app:projection}\label{app:rescaling}\label{app:clamp4b}
\noindent\textbf{\rnow{Full component ablation on SDXL.}} \rnow{Table~\ref{tab:ablation_full} adds the preference metrics to Table~\ref{tab:ablation} for SDXL+CFG, together with the sparse-attention-only update. The full residual raises GenEval but lowers PickScore; the parallel component restores PickScore, and the norm clamp brings the largest CLIPScore and ImageReward gains. The gate raises GenEval further while keeping the preference metrics of the ungated update.}
\begin{table}[h!]
\centering\small
\caption{\rnow{Component ablation on SDXL+CFG (50 steps).}}
\label{tab:ablation_full}
\begin{tabular}{lcccc}
\toprule
Configuration & GE~$\uparrow$ & CS~$\uparrow$ & IR~$\uparrow$ & PS~$\uparrow$ \\
\midrule
Base & 0.545 & 26.45 & 0.604 & 22.31 \\
+\,sparse attention & 0.581 & 26.49 & 0.651 & 22.29 \\
+\,full residual & 0.594 & 26.54 & 0.732 & 22.11 \\
+\,parallel only & 0.601 & 26.68 & 0.715 & 22.34 \\
+\,norm clamp (w/o gate) & 0.607 & 27.26 & 0.812 & 22.40 \\
+\,gate (ReGA) & 0.619 & 27.44 & 0.822 & 22.39 \\
\bottomrule
\end{tabular}
\end{table}

\noindent\textbf{Retained orthogonal fraction $\zeta$.} On SDXL+CFG, every metric decreases monotonically as $\zeta$ grows from 0 to 1, and guiding along $r^{\perp}$ alone is worst on all of them (Table~\ref{tab:zeta_full}).
\begin{table}[h!]
\centering\small
\caption{$\zeta$ ablation (SDXL+CFG, $\lambda{=}10$, clamp on, every step). The $\zeta{=}0$ row is the w/o-gate row of Table~\ref{tab:wogate_full}.}
\label{tab:zeta_full}
\begin{tabular}{lccccc}
\toprule
$\zeta$ & FID~$\downarrow$ & CS~$\uparrow$ & IR~$\uparrow$ & PS~$\uparrow$ & HPSv2~$\uparrow$ \\
\midrule
\rowcolor{green!10} 0.00 (parallel) & \bf 14.63 & \bf 27.26 & \bf 0.812 & \bf 22.40 & \bf 0.271 \\
0.25 & 15.35 & 27.13 & 0.789 & 22.31 & 0.268 \\
0.50 & 15.56 & 27.01 & 0.763 & 22.15 & 0.270 \\
0.75 & 15.71 & 26.82 & 0.755 & 22.11 & 0.265 \\
1.00 (full) & 15.92 & 26.63 & 0.735 & 22.10 & 0.262 \\
Orthogonal only & 19.23 & 25.25 & 0.415 & 21.79 & 0.260 \\
\bottomrule
\end{tabular}
\end{table}

\noindent\textbf{Projection target.} Theorem~\ref{thm:directional} certifies the full residual $r(\mbx)$; the projection onto $\Tca(\mbx)$ is a choice, and we compare it with alternatives that change only the projection target (SDXL, GenEval 553 prompts, identical seeds, $\lambda{=}3$ for this controlled ablation): value-centered subtracts the mean value vector, $\Tca-\bar{\Vb}$; dominant token uses the value vector of the most-attended token, either globally or per query. Each alternative replaces $\mathbf{P}_\alpha$ in the coefficient of Definition~\ref{def:decomp} by the probability-space preimage of the stated direction, so all four rows differ only in the projection target. Our choice is best on all four metrics (Table~\ref{tab:projection}).
\begin{table}[h!]
\centering\small
\caption{Projection-target ablation (SDXL, GenEval 553-prompt subset, $\lambda{=}3$; not comparable with Table~\ref{tab_main}).}
\label{tab:projection}
\begin{tabular}{lcccc}
\toprule
Projection target & GE~$\uparrow$ & CS~$\uparrow$ & IR~$\uparrow$ & PS~$\uparrow$ \\
\midrule
\rowcolor{green!10} $\Tca$ (ours) & \bf 0.615 & \bf 28.99 & \bf 0.833 & \bf 22.85 \\
value-centered ($\Tca-\bar{\Vb}$) & 0.612 & 28.91 & 0.821 & 22.82 \\
dominant token (global) & 0.603 & 28.92 & 0.815 & 22.84 \\
dominant token (per query) & 0.596 & 28.83 & 0.786 & 22.56 \\
\bottomrule
\end{tabular}
\end{table}

\noindent\textbf{Norm clamp.} At high guidance scale the parallel component tends to grow in magnitude, so the extrapolated update can drift outside the typical norm range of the attention output. Following the norm rescaling of APG~\citep{apg}, the discrepancy of a layer is clamped to Frobenius norm $\eta$ before the projection, $\boldsymbol{\Delta}\leftarrow\min(1,\eta/\norm{\boldsymbol{\Delta}}_F)\,\boldsymbol{\Delta}$ in~\eqref{eq:gag}; since the projection is linear, this scales $r^{\parallel}$ and $r^{\perp}$ by the same factor and preserves their directions and the decomposition of Definition~\ref{def:decomp}. The clamp acts on the norm of the whole attention tensor of a layer, not on a per-query quantity, so $\eta$ is not comparable with the per-query magnitudes of Appendix~\ref{app:diag}. On SDXL, removing it costs most of the ImageReward gain and part of the CLIPScore gain at $\lambda{=}10$ (parallel-only versus norm-clamp rows of Table~\ref{tab:ablation_full}). On FLUX.2-4B, removing it lowers PickScore, ImageReward, and HPSv2 significantly, as does replacing the normalized update with the plain extrapolation $P{+}5(P{-}D)$ (Table~\ref{tab:clamp4b}). Projecting onto the dense map, adding a random orthogonal vector, or keeping part of $r^{\perp}$ leaves GenEval within half a point of ReGA, while projecting onto a uniform map lowers it significantly.
\begin{table}[h!]
\centering\small
\caption{Ablations on FLUX.2-4B. GE: GenEval (553 prompts $\times$ 4). $\Delta$PS, $\Delta$IR, $\Delta$HPS: change from ReGA on 500 MS-COCO captions, paired. $P$, $D$: sparse and dense attention maps. $^{\ast}$: 95\% CI excludes zero.}
\label{tab:clamp4b}
\begin{minipage}[t]{0.58\linewidth}\centering
\begin{tabular}[t]{lcccc}
\toprule
Variant & GE~$\uparrow$ & $\Delta$PS & $\Delta$IR & $\Delta$HPS \\
\midrule
w/o norm clamp & .831 & $-$.28$^{\ast}$ & $-$.13$^{\ast}$ & $-$.015$^{\ast}$ \\
plain $P{+}5(P{-}D)$ & .842 & $-$.10$^{\ast}$ & $-$.04$^{\ast}$ & $-$.003$^{\ast}$ \\
\bottomrule
\end{tabular}
\end{minipage}\hfill
\begin{minipage}[t]{0.4\linewidth}\centering
\begin{tabular}[t]{lc}
\toprule
Variant & GE~$\uparrow$ \\
\midrule
\rowcolor{green!10} ReGA & .843 \\
projection onto dense map & .841 \\
projection onto uniform map & .832$^{\ast}$ \\
random $r^{\perp}$ added & .844 \\
$\zeta=0.5$ / $\zeta=1$ & .838 / .840 \\
\bottomrule
\end{tabular}
\end{minipage}
\end{table}

\begin{figure}[p]\centering\includegraphics[width=497bp]{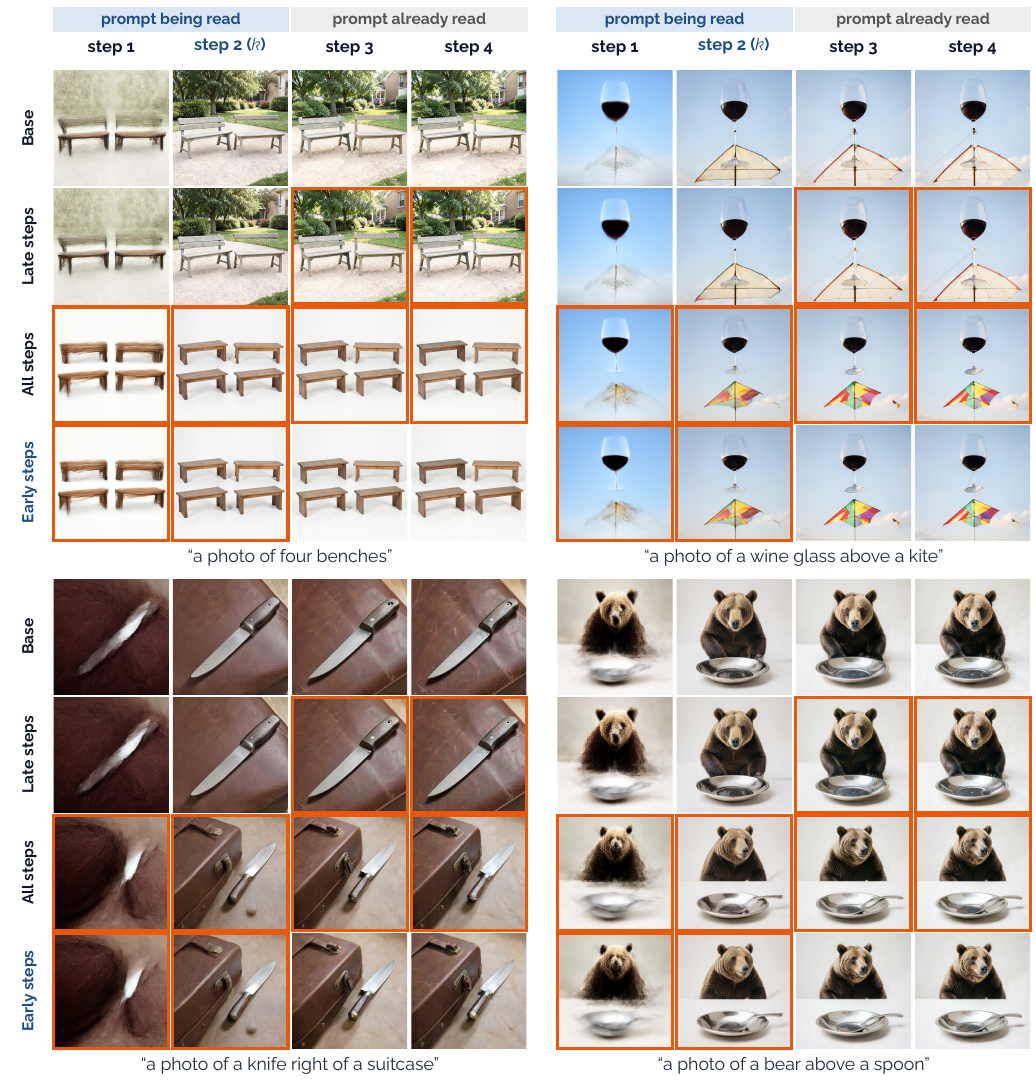}\setcounter{figure}{10}\caption{\rnow{\textbf{Predicted $\hat{\mbx}_0$ along sampling on FLUX.2-4B.} Orange frames mark the guided steps; the prompt is read up to step $k$.}}\label{fig:x0_4b}\end{figure}\setcounter{figure}{12}
\FloatBarrier
\subsection{Comparison with NAG}\label{app:nag}
NAG~\citep{nag} contrasts attention under different conditioning \rnow{and adds a second attention computation in every guided layer}. On Hyper-SDXL, 4-step ReGA is above NAG on every metric, and applying ReGA on top of NAG degrades every metric except PickScore, as Section~\ref{sec:directional} predicts for operators that do not share conditioning (Table~\ref{tab:nag_full}).
\begin{table}[h!]
\centering\small
\caption{NAG comparison on Hyper-SDXL.}
\label{tab:nag_full}
\begin{tabular}{lccccccc}
\toprule
Method & NFE & GE~$\uparrow$ & CS~$\uparrow$ & IR~$\uparrow$ & PS~$\uparrow$ & HPSv2~$\uparrow$ & FID~$\downarrow$ \\
\midrule
Hyper-SDXL & 4 & 0.556 & 26.09 & 0.959 & 22.87 & 0.313 & 21.23 \\
+\,PLADIS & 4 & 0.589 & 27.11 & 1.060 & 22.91 & 0.315 & 19.83 \\
\rowcolor{green!10} +\,ReGA & 4 & \bf 0.594 & \bf 27.55 & \bf 1.071 & \bf 22.93 & \bf 0.317 & \bf 19.12 \\
+\,NAG & ${>}4$ & 0.588 & 27.10 & 1.046 & 22.86 & 0.315 & 19.92 \\
+\,ReGA on top of NAG & ${>}4$ & 0.571 & 26.64 & 0.971 & 22.88 & 0.312 & 20.83 \\
+\,PLADIS & 8 & 0.592 & 27.21 & 1.181 & 23.10 & 0.317 & 19.35 \\
\rowcolor{green!10} +\,ReGA & 8 & \bf 0.602 & \bf 27.78 & \bf 1.213 & \bf 23.13 & \bf 0.319 & \bf 18.81 \\
\bottomrule
\end{tabular}
\end{table}

\FloatBarrier
\FloatBarrier
\subsection{SDXL Diagnostics}\label{sec:regime}\label{app:diag}\label{app:injection}
\noindent\textbf{Protocol.} SDXL-base-1.0, 250 GenEval prompts stratified over the six categories (125 single-object, 125 compositional), seed 42, 50 steps, CFG scale 5.0, all 70 cross-attention layers and all heads. Measurements are taken in observe mode: the sampled trajectory is the baseline softmax trajectory and the logged quantities do not intervene (pixel-identical outputs verified). Per query we record $\cos(r,\Tca)$ with $r=\Tca-\Tcd$, the top-1 versus top-2 logit margin, the $\alpha\texttt{-Entmax}$ support size (number of tokens with nonzero attention), and $\norm{r^{\parallel}}/\norm{r^{\perp}}$; tables report medians over queries and then over prompts, layers, heads, and steps within each cell, for $16.25$M (prompt$\times$layer$\times$head$\times$step) measurements in total.

\noindent\textbf{Assumptions in multi-head cross-attention.} The analysis of Section~\ref{sec:directional} is stated for one retrieval operator; in a diffusion backbone every head of every cross-attention layer is one such operator, with the text-token values as stored patterns and each spatial query as the input state, and $\beta=1/\sqrt{d}$ fixed by the architecture. The projection of Definition~\ref{def:decomp} is taken over the attention tensor of a layer, whereas the diagnostics are resolved per head, so they test whether each head is in a sharp, well-separated regime, which is what the logit-margin and support-size measurements report.

\begin{figure}[h!]\centering\includegraphics[width=497bp]{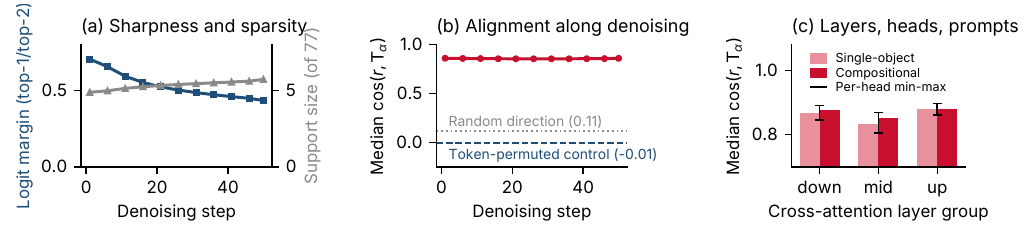}\caption{SDXL cross-attention relative to Theorem~\ref{thm:directional}: (a) sharp but not collapsed retrieval, (b) aligned residual versus controls (SDXL detail of Figure~\ref{fig:motivation}a), (c) alignment across layers, heads, and prompts.}\label{fig:diag}\end{figure}

\noindent\textbf{Retrieval is sharp but below the collapse threshold (Theorem~\ref{thm:directional}(i)).} Since $\beta=1/\sqrt{d}$ is fixed, sharpness is governed by the logit margin and the $\alpha\texttt{-Entmax}$ support. The top-1 pattern leads by $1.8\times$ early and $1.5\times$ late while the support stays at 4--7 of 77 tokens (Figure~\ref{fig:diag}a; Table~\ref{tab:diag_margin}). Along the trajectory the pooled margin falls by $38\%$ (top-1/top-2 ratio $1.78\to1.50$) and the support grows by $17\%$ ($4.9\to5.7$ tokens): retrieval weakens along denoising but remains sparse and top-1 dominant, below the single-token collapse of Theorem~\ref{thm:directional}, so alignment is a matter of measurement.

\noindent\textbf{The residual points along retrieval (Theorem~\ref{thm:directional}(ii)).} Since the stored pattern $\xib_\mu$ is not identified inside a trained layer, we measure the alignment with the sparse retrieval direction $\Tca$, which (i) places closer to $\xib_\mu$. The median $\cos(r,\Tca)$ is $0.86$ in probability space at every step, and none of the $16.25$M measurements is non-positive ($2.5\%$ are below $0.7$); in retrieval space ($\mt=\mbp\Vb$) the medians are $0.75$--$0.81$ with $0.08\%$ non-positive values, and we report the probability-space quantity because it is the one ReGA operates on (Figure~\ref{fig:diag}b). A random direction in $\mathbb{R}^{77}$ has expected $|\cos|\approx1/\sqrt{77}\approx0.11$, and a token-permuted control, which keeps the sparsity pattern but permutes which tokens it points to, gives $-0.008$. The alignment thus reflects where retrieval points rather than sparsity, and it stays flat ($0.861\to0.860$) while sharpness decays. This is a consistency check on (ii) rather than a test of its condition, which involves $\xib_\mu$.

\noindent\textbf{Alignment is uniform across layers, heads, and prompt types.} The alignment is positive in every layer group and every head (Figure~\ref{fig:diag}c; Tables~\ref{tab:diag_align} and~\ref{tab:diag_heads}). Per-layer medians span $0.767$ to $0.924$; the lower values concentrate in the high-resolution blocks and the higher values in the low-resolution blocks, a resolution-dependent pattern rather than noise. Compositional prompts align slightly better than single-object prompts ($0.874$ versus $0.868$, Mann--Whitney $p=2\times10^{-5}$), with attribute binding ($0.884$) and two-object ($0.877$) the strongest categories. $r^{\parallel}$ dominates $r^{\perp}$ throughout, by a factor of $1.50$--$1.88$ (Figure~\ref{fig:resmag}).

\begin{table}[h!]
\centering\small
\caption{Median $\cos(r,\Tca)$ by denoising-step band and layer group (single-object / compositional prompts).}
\label{tab:diag_align}
\begin{tabular}{lccc}
\toprule
Denoising steps & down & mid & up \\
\midrule
early (1--17) & 0.869 / 0.879 & 0.836 / 0.856 & 0.879 / 0.883 \\
middle (18--34) & 0.865 / 0.874 & 0.834 / 0.853 & 0.876 / 0.879 \\
late (35--50) & 0.867 / 0.875 & 0.832 / 0.851 & 0.882 / 0.881 \\
\bottomrule
\end{tabular}
\end{table}

\begin{table}[h!]
\centering\small
\begin{minipage}[t]{0.42\linewidth}\centering
\caption{Per-head median $\cos(r,\Tca)$. No head has a non-positive median; within a cell of Table~\ref{tab:diag_align} the heads differ by at most $0.021$.}
\label{tab:diag_heads}
\begin{tabular}{lcc}
\toprule
Group & across heads & per-head range \\
\midrule
down & 0.875--0.877 & 0.845--0.891 \\
mid & 0.845--0.848 & 0.805--0.869 \\
up & 0.877--0.882 & 0.861--0.897 \\
\bottomrule
\end{tabular}
\end{minipage}\hfill
\begin{minipage}[t]{0.54\linewidth}\centering
\caption{Median top-1 versus top-2 logit margin (single-object / compositional). Its exponential is the top-1/top-2 probability ratio.}
\label{tab:diag_margin}
\begin{tabular}{lccc}
\toprule
Steps & down & mid & up \\
\midrule
early & 0.578 / 0.520 & 0.409 / 0.352 & 0.367 / 0.338 \\
middle & 0.439 / 0.396 & 0.369 / 0.308 & 0.322 / 0.295 \\
late & 0.405 / 0.359 & 0.346 / 0.288 & 0.297 / 0.274 \\
\bottomrule
\end{tabular}
\end{minipage}
\end{table}

\noindent\textbf{Direction versus scale of $r^{\perp}$ (Proposition~\ref{thm:conv_orth}).} Proposition~\ref{thm:conv_orth} attributes the stability of $\zeta{=}0$ to Assumption~\ref{ass:weak_orth}, whose attractor is unobservable: the exact orthogonal error $u_t$ requires $\mbx^*$, and the observable proxy $\norm{r^{\perp}}$ shows no trend along the trajectory ($0.1345\to0.1334$; Figure~\ref{fig:resmag}). The observable test is whether the harm of $r^{\perp}$ lies in its direction or its scale. We run a norm-matched injection at $\zeta{=}0.5$, replacing the actual $r^{\perp}$ by a random vector orthogonal to $\mathbf{P}_\alpha$ in the space of Definition~\ref{def:decomp}, at the norm of $r^{\perp}$. The random injection leaves every metric unchanged, while the actual $r^{\perp}$ at the same norm degrades all of them (Figure~\ref{fig:residual}), so $\zeta{=}0$ removes a harmful direction rather than excess magnitude, and Assumption~\ref{ass:weak_orth} is best read as a constraint on the direction of the orthogonal error (Remark~\ref{rem:stability}).

\begin{figure}[h!]
\centering
\begin{minipage}[b]{0.40\linewidth}\centering
\includegraphics[width=201bp]{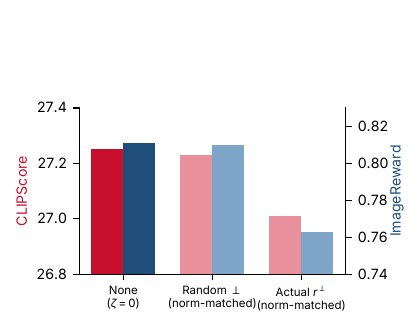}
\caption{Norm-matched injection of $r^{\perp}$ on SDXL+CFG.}\label{fig:residual}
\end{minipage}\hspace{0.08\linewidth}
\begin{minipage}[b]{0.34\linewidth}\centering
\includegraphics[width=172bp]{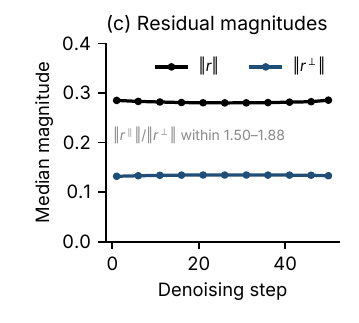}
\caption{Residual magnitudes $\norm{r}$ and $\norm{r^{\perp}}$ along denoising (SDXL).}\label{fig:resmag}
\end{minipage}
\end{figure}

\subsection{MM-DiT Diagnostics}\label{app:mmdit_diag}
Figure~\ref{fig:mmdit_diag} repeats the SDXL diagnostics of Figure~\ref{fig:diag} on the image-to-text block of FLUX.2-4B (250 prompts, five double-stream blocks): the alignment of $r$ with $\Tca$ is the same at every step and block, no head has a non-positive median, and retrieval is sharp but not collapsed. Figure~\ref{fig:inj4b} repeats the injection test of Figure~\ref{fig:residual}: on FLUX.2-4B, adding $r^{\perp}$ or a random orthogonal vector to $r^{\parallel}$ changes GenEval by less than a significant margin, and $r^{\perp}$ alone stays at the base level.

\begin{figure}[h!]\centering\includegraphics[width=497bp]{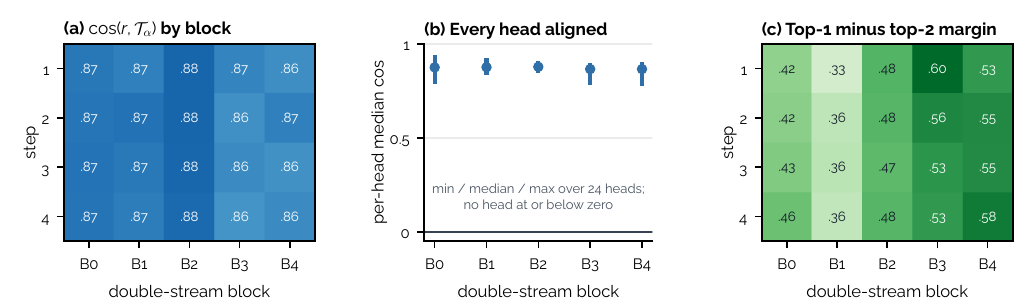}\caption{FLUX.2-4B diagnostics: (a) alignment by step and block, (b) per-head median alignment, (c) logit margin.}\label{fig:mmdit_diag}\end{figure}
\begin{figure}[h!]\centering\includegraphics[width=225bp]{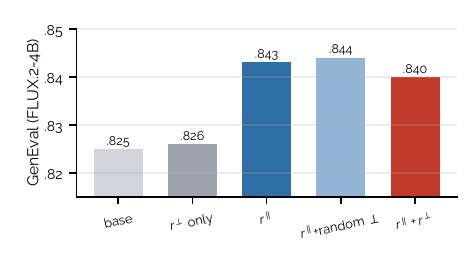}\caption{Orthogonal injection on FLUX.2-4B (GenEval).}\label{fig:inj4b}\end{figure}

\subsection{DDIM Inversion-Reconstruction}\label{app:ddim}
Because the target of a reconstruction is known, DDIM inversion tests the convergence-rate ordering of Theorem~\ref{thm:directional}(i), $\norm{\Tca(\mbx)-\mbx^*}\le\norm{\Tcd(\mbx)-\mbx^*}$, directly.

\noindent\textbf{Setup.} We sample $200$ prompts at random from the MS-COCO validation set and use the corresponding real images as ground-truth attractors $\mbx^{*}$, on SDXL~\citep{sdxl}. \emph{Inversion} ($\mbx^{*}\to\mbz$): for each real image we obtain the inverted latent $\mbz$ by $50$-step DDIM inversion~\citep{song2021ddim} without CFG, which isolates the deterministic ODE trajectory so that the reconstruction error reflects only the retrieval operator under test. \emph{Reconstruction} ($\mbz\to\hat{\mbx}$): from the same latent we run $50$-step DDIM denoising with each retrieval operator, Dense ($\alpha{=}1.0$), Sparse ($\alpha{=}1.5$), and ReGA ($\zeta{=}0$, $\eta{=}15$), with all other hyperparameters and schedules fixed. We score each $\hat{\mbx}$ against $\mbx^{*}$ with SSIM~\citep{wang2004image} and LPIPS~\citep{zhang2018unreasonable}, averaged over the $200$ prompts.

\noindent\textbf{Results.} Sparse improves SSIM and LPIPS over Dense, and ReGA improves them further (Table~\ref{tab:ddim}), consistent with the theoretical ordering and indicating that the decomposition keeps the convergence-aligned component of $r(\mbx)$. ReGA preserves fine structure such as text, sharp edges, and small objects more faithfully than Dense or Sparse retrieval alone (Figure~\ref{fig:ddim}).

\begin{table}[h!]
\centering\small
\caption{DDIM inversion-reconstruction (SDXL).}
\label{tab:ddim}
\begin{tabular}{lcc}
\toprule
Method & SSIM~$\uparrow$ & LPIPS~$\downarrow$ \\
\midrule
Dense ($\alpha{=}1.0$) & 0.7461 & 0.2091 \\
Sparse ($\alpha{=}1.5$) & 0.7545 & 0.1900 \\
\rowcolor{green!10} \textbf{ReGA} & \bf 0.7565 & \bf 0.1838 \\
\bottomrule
\end{tabular}
\end{table}

\begin{figure}[h!]\centering\includegraphics[width=472bp]{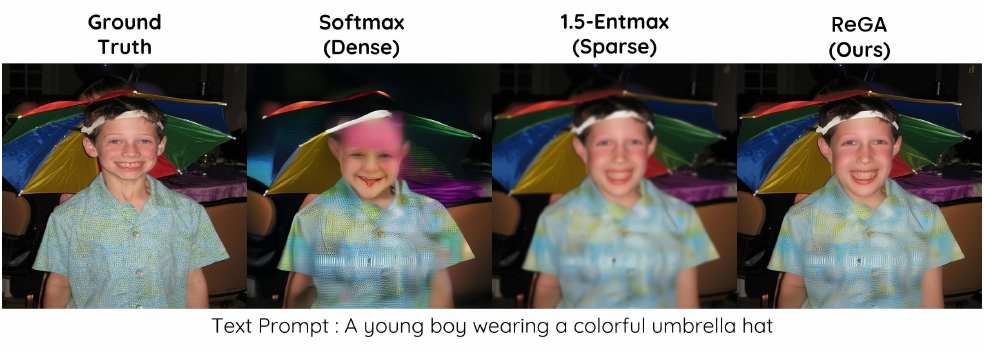}\caption{DDIM inversion-reconstruction on SDXL. All reconstructions start from the same $50$-step DDIM-inverted latent.}\label{fig:ddim}\end{figure}

\section{User Study Protocol}\label{app:userstudy}
The study was a blinded three-way comparison with 35 participants on 50 random prompts (SDXL and FLUX). Each trial showed the three images (baseline, PLADIS, ReGA) simultaneously, labeled A/B/C with randomized assignment, and participants picked the best image per criterion (text alignment, visual quality). Participants were unpaid volunteers recruited through an anonymous institutional link, blind to the paper and the methods. Preference rates are reported with 95\% bootstrap confidence intervals over trials.

\section{Additional Qualitative Results}\label{sec:visual}
All comparisons within a row use identical prompts and seeds, so visual differences reflect the guidance rather than sampling stochasticity.

\noindent\textbf{Existing guidance.} Figures~\ref{fig:fig_cfg_1}, \ref{fig:fig_apg_1}, and \ref{fig:fig_pag_1} compare CFG~\citep{CFG}, APG~\citep{apg}, and PAG~\citep{PAG} on SDXL (50 steps) with their PLADIS- and ReGA-augmented counterparts. Relative to PLADIS, ReGA yields better text--image alignment and detail on complex prompts (object counting, visual-effect rendering) at the same cost.

\noindent\textbf{Step-distilled models.} Figures~\ref{fig:fig_hyper} and \ref{fig:fig_dmd2} show Hyper-SDXL~\citep{hyper} and DMD2~\citep{dmd2} in the 4-step regime. ReGA reduces visible artifacts and improves prompt--image coherence over PLADIS at nearly the same cost, fitting step-distilled pipelines where dual-pass guidance is unavailable.

\noindent\textbf{FLUX.1.} Figures~\ref{fig:fig_schnell} and~\ref{fig:fig_dev} show FLUX.1-Schnell (4 steps) and FLUX.1-Dev (28 steps): ReGA transfers from the UNet of SDXL to MM-DiT in both few-step and multi-step regimes.

\noindent\textbf{Comparison with NAG.} Figure~\ref{fig:fig_nag} compares NAG~\citep{nag}, PLADIS, and ReGA on Hyper-SDXL (4 steps). ReGA scores above NAG on every metric in this setting (Appendix~\ref{app:nag}).

\noindent\textbf{MM-DiT backbones with the gate.} Figure~\ref{fig:qual_mmdit} collects cases on FLUX.1-Schnell, Qwen-Image, FLUX.2-4B, and FLUX.2-9B where the base sampler fails: ReGA separates bound objects, places them as the prompt specifies, fixes counts, and corrects rendered text, while keeping the composition of the base sample.

\begin{figure}[p]\centering\includegraphics[width=376bp]{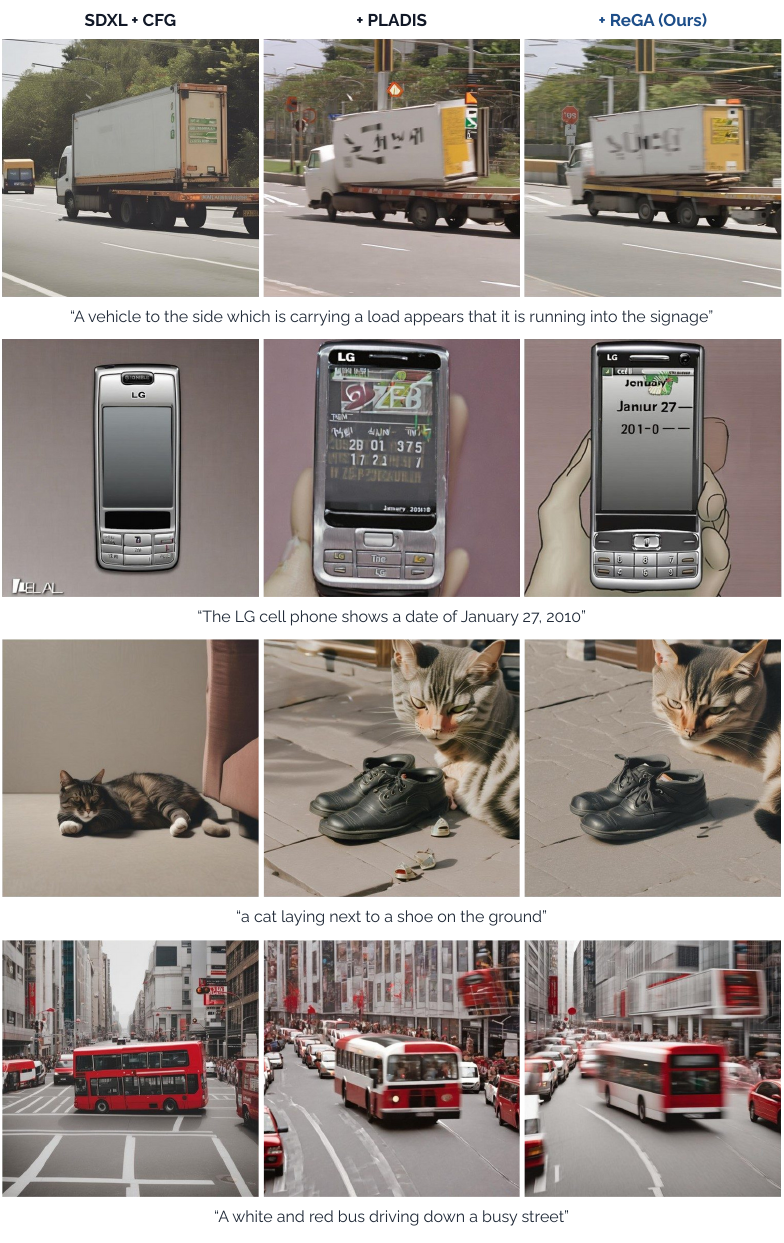}\caption{\textbf{Qualitative comparison on SDXL with CFG.} CFG~\citep{CFG} alone, with PLADIS, and with ReGA (50 steps).}\label{fig:fig_cfg_1}\end{figure}
\begin{figure}[p]\centering\includegraphics[width=376bp]{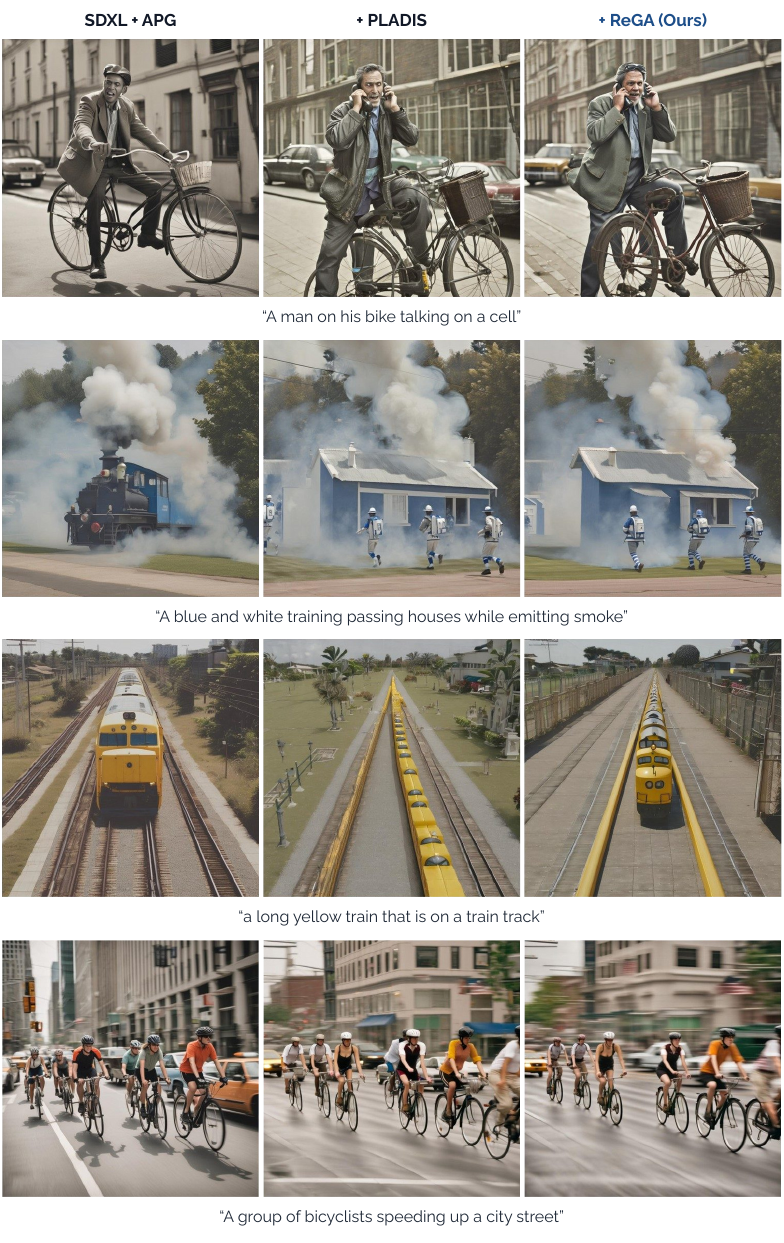}\caption{\textbf{Qualitative comparison on SDXL with APG.} APG~\citep{apg} alone, with PLADIS, and with ReGA (50 steps).}\label{fig:fig_apg_1}\end{figure}
\begin{figure}[p]\centering\includegraphics[width=376bp]{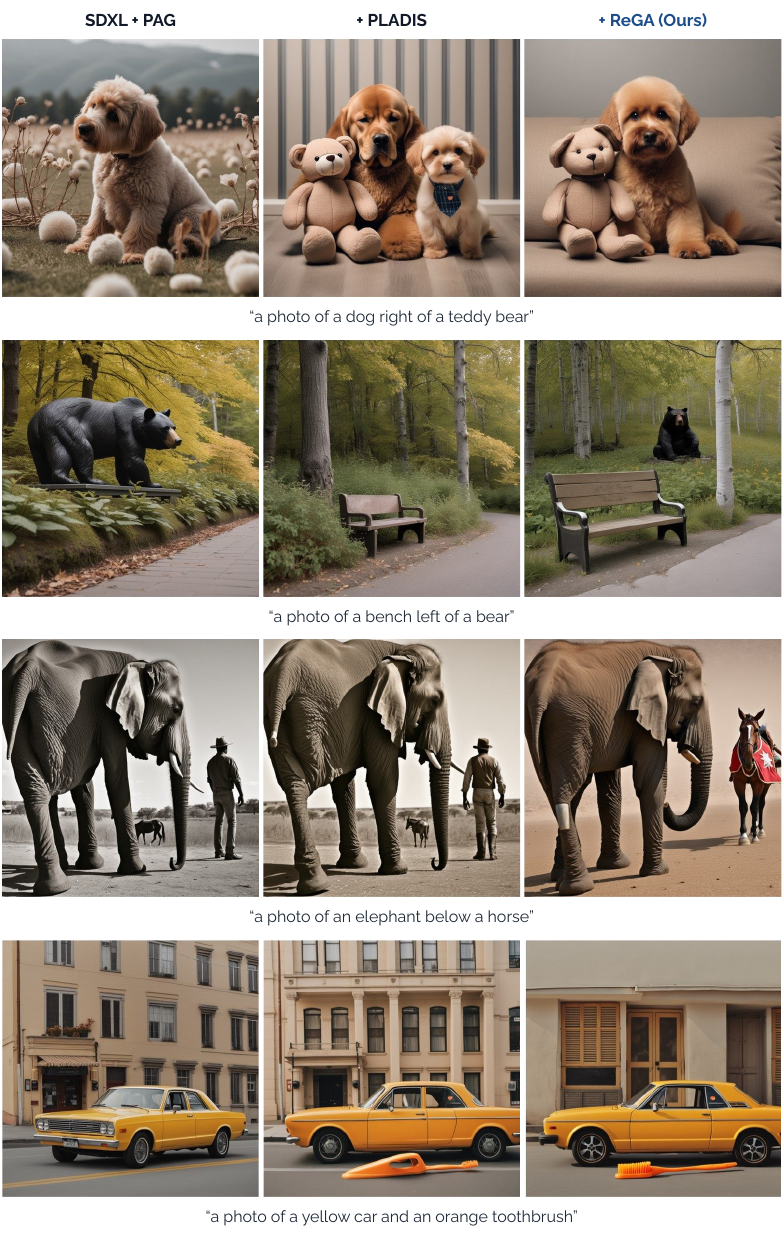}\caption{\textbf{Qualitative comparison on SDXL with PAG.} PAG~\citep{PAG} alone, with PLADIS, and with ReGA.}\label{fig:fig_pag_1}\end{figure}
\begin{figure}[p]\centering\includegraphics[width=376bp]{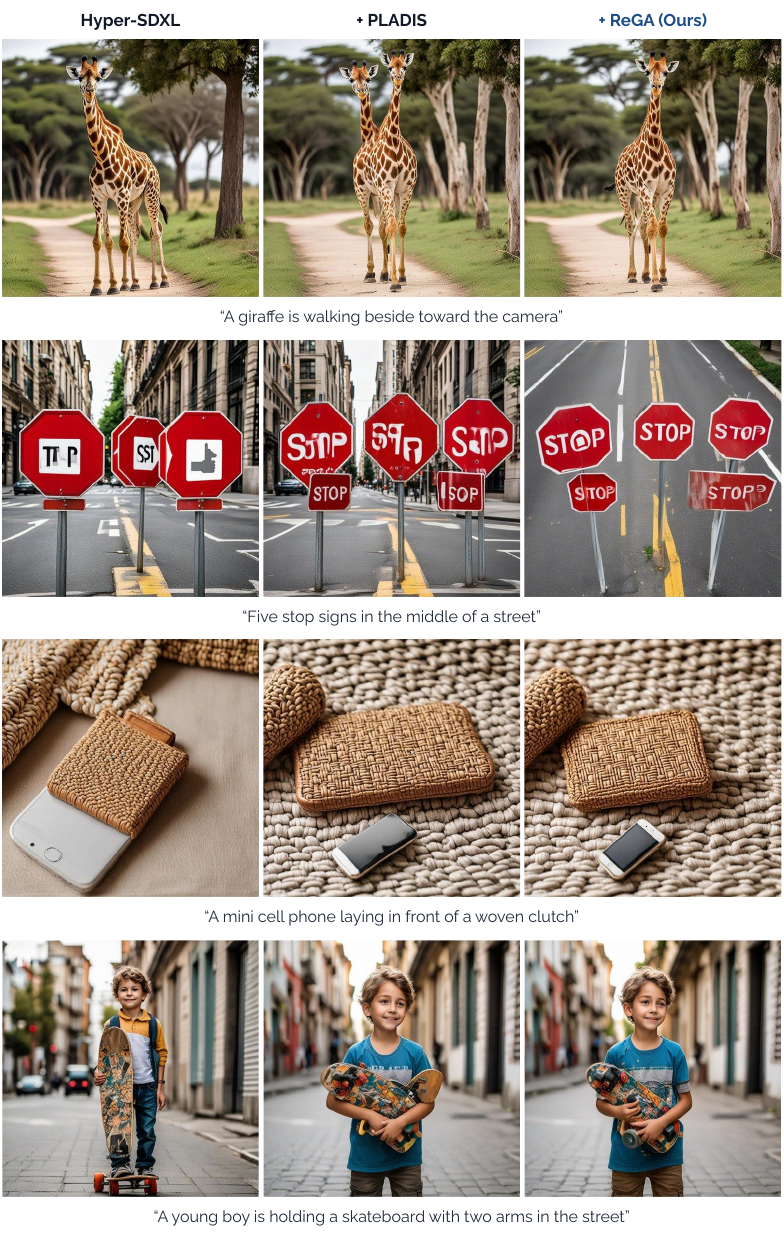}\caption{\textbf{Qualitative comparison on Hyper-SDXL.} Four-step samples from the base sampler~\citep{hyper}, PLADIS, and ReGA.}\label{fig:fig_hyper}\end{figure}
\begin{figure}[p]\centering\includegraphics[width=376bp]{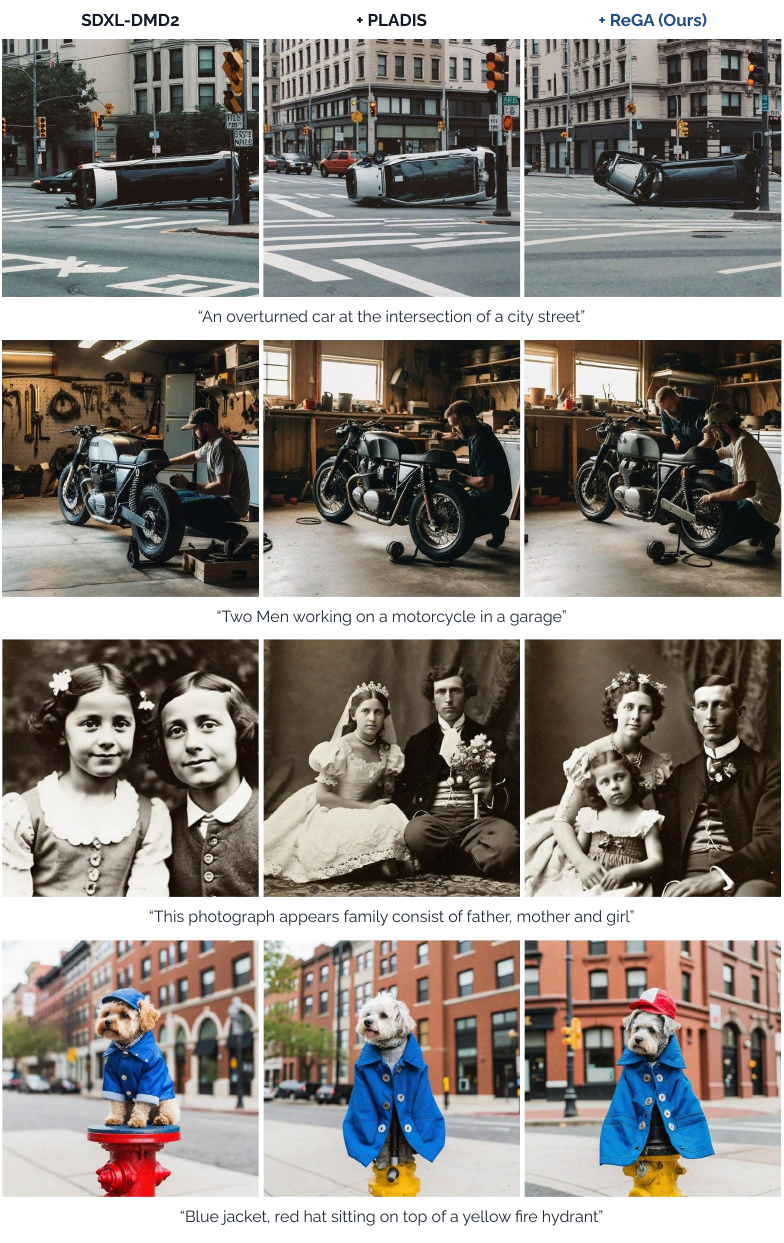}\caption{\textbf{Qualitative comparison on SDXL-DMD2.} Four-step samples from the base sampler~\citep{dmd2}, PLADIS, and ReGA.}\label{fig:fig_dmd2}\end{figure}
\begin{figure}[p]\centering\includegraphics[width=376bp]{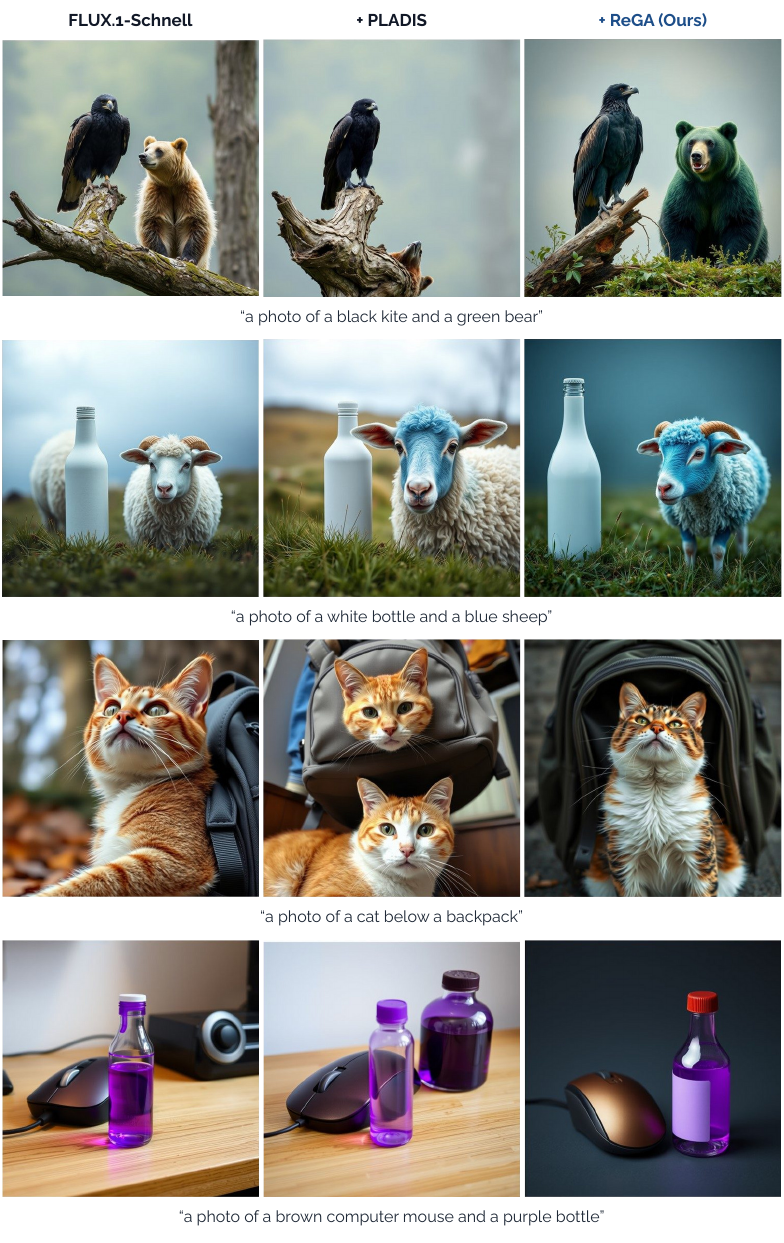}\caption{\textbf{Qualitative comparison on FLUX.1-Schnell.} Four-step samples from the base sampler, PLADIS, and ReGA.}\label{fig:fig_schnell}\end{figure}
\begin{figure}[p]\centering\includegraphics[width=376bp]{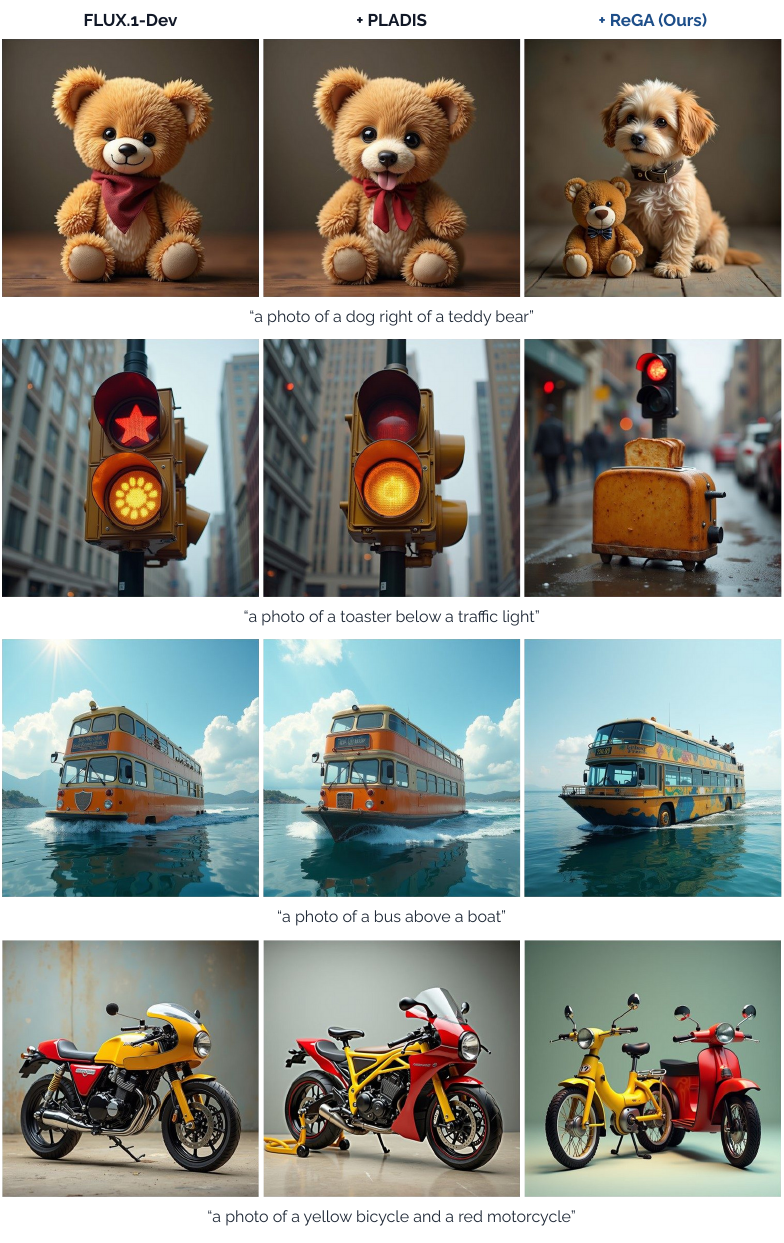}\caption{\textbf{Qualitative comparison on FLUX.1-Dev.} 28-step samples from the base sampler, PLADIS, and ReGA.}\label{fig:fig_dev}\end{figure}
\begin{figure}[p]\centering\includegraphics[width=497bp]{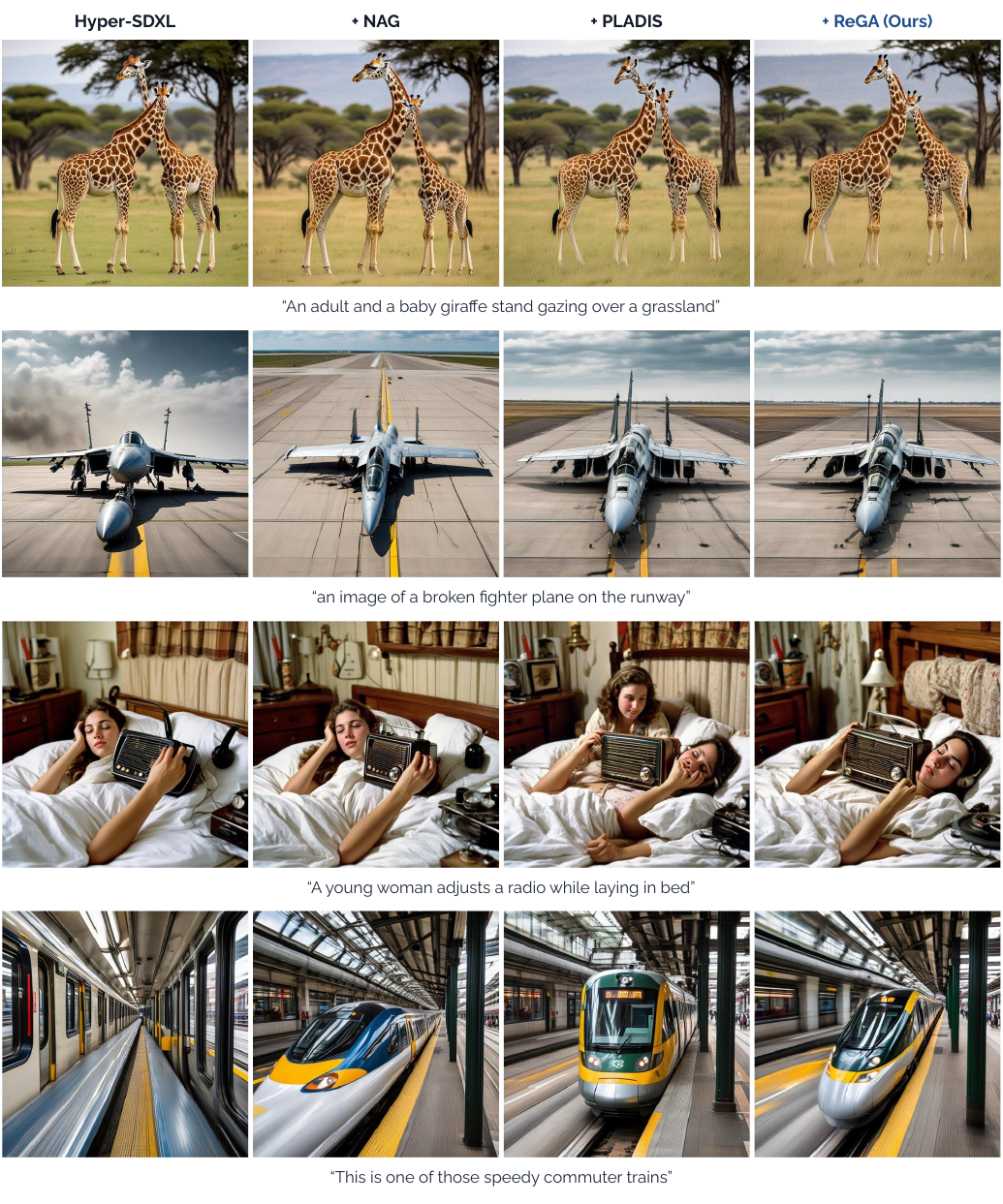}\caption{\textbf{Comparison with NAG on Hyper-SDXL.} Four-step samples from the base sampler, NAG~\citep{nag}, PLADIS, and ReGA.}\label{fig:fig_nag}\end{figure}
\begin{figure}[p]\centering\includegraphics[width=471bp]{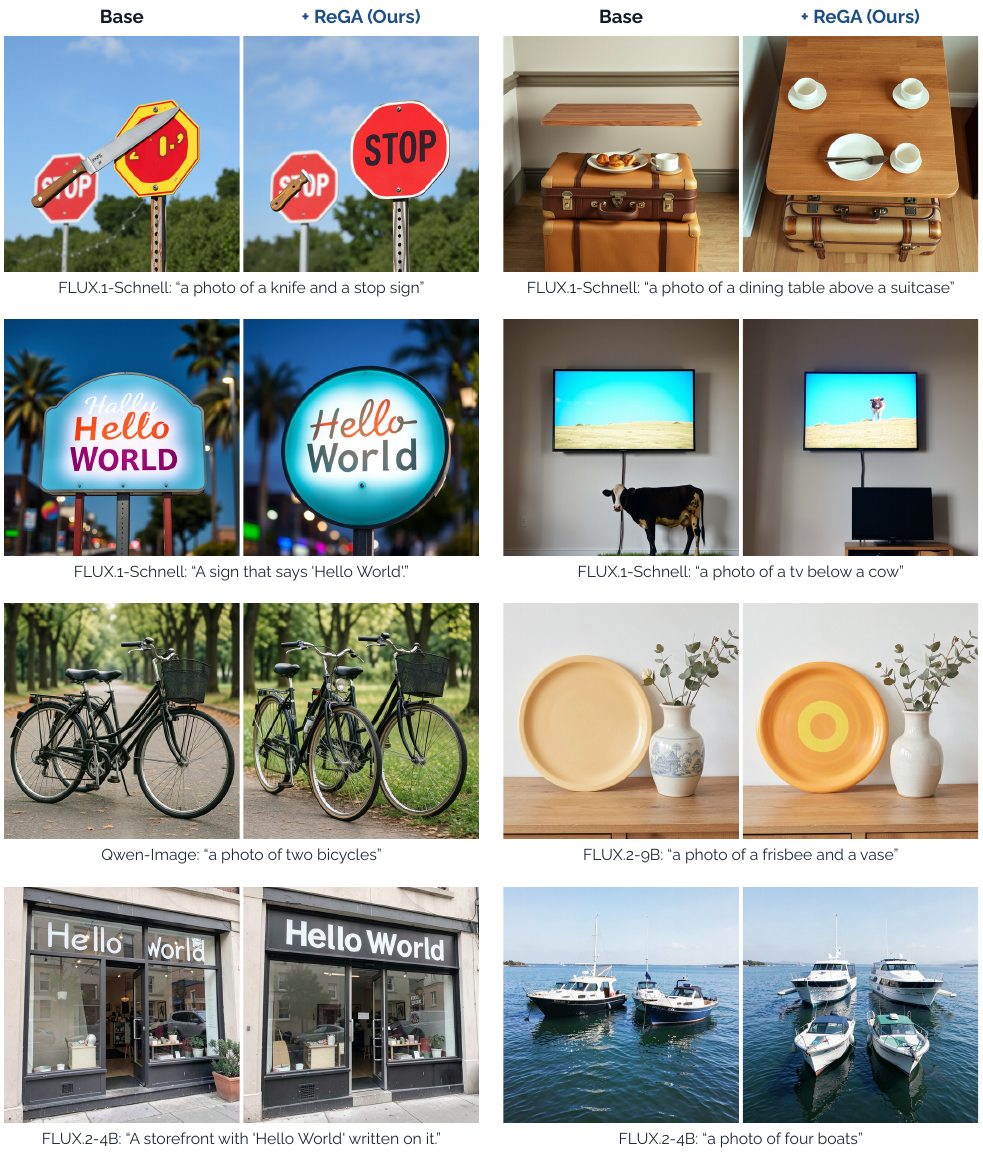}\caption{\textbf{Qualitative comparison on MM-DiT backbones.} Cases where the base sampler fails a counting, position, attribute, or text prompt.}\label{fig:qual_mmdit}\end{figure}

\clearpage
\section{Limitations and Discussion}\label{app:discussion}
\noindent\textbf{Limitations.} ReGA applies to attention pairs that share queries, keys, values, and conditioning, so contrasts across prompts such as NAG lie outside its analysis. The guided window is read once per backbone from the base sampler, and the retained fraction of prompt-directed guidance is a single value shared by all backbones. $\alpha{=}1.5$ and the layer placement follow PLADIS. Video and other multimodal diffusion models share this attention structure but are not evaluated here.

\noindent\textbf{Model-level vs.\ layer-level Hopfield analyses.} A growing line of work connects diffusion models with associative memory at the \emph{model level}, treating the entire denoising network or score function as an associative-memory system. \citet{ambrogioni2023inthopfieldnet} shows that diffusion models can be interpreted as energy-based models whose energy, when trained on discrete patterns, is asymptotically identical to that of a modern Hopfield network. \citet{hoover2023memory} survey the resemblances between associative memories and diffusion models, describing diffusion models as a kind of associative memory whose stability guarantees are bypassed by the noise and step-size schedules. \citet{pham2025memorization} show that, on small datasets, diffusion models create a distinct attractor for each training sample, as dense associative memories do below their storage capacity, and move from memorization to generalization as the data grows. In contrast, \citet{hopfield} operates at the \emph{layer level}: the update rule of a modern Hopfield network is equivalent to Transformer attention, with the stored patterns giving the keys and values and the state patterns giving the queries. Our analysis builds on this layer-level reading: ReGA is not a model-level Hopfield interpretation of diffusion, but a directional analysis of one retrieval pair ($\Tca$ vs.\ $\Tcd$) inside an attention layer. Theorem~\ref{thm:directional} is therefore complementary to model-level work and does not depend on whether the full diffusion network is itself viewable as a Hopfield system.

\noindent\textbf{Why only the image-to-text block.} ReGA acts on image queries attending to text keys: the cross-attention layers of SDXL and the image-to-text block of joint attention on MM-DiT. There each text token is a stored pattern and each image query retrieves the most relevant one, which is the reading of the MHN analysis. ReGA does not act on image-to-image attention, the self-attention of SDXL or the image-to-image block of joint attention, which encodes spatial structure; sparsifying it has been reported to limit interactions between pixels and degrade generation~\citep{pladis}. Extending the directional analysis to image-to-image attention is left to future work.

\section*{Reproducibility and Ethics}\phantomsection\label{app:repro}
\noindent\textbf{Reproducibility.} The sampling protocol, per-backbone settings, and guided windows are in Appendices~\ref{sec:detail} and~\ref{app:settings}, the implementation in Appendix~\ref{app:overhead}, and the measurement protocol of the diagnostics in Appendix~\ref{app:diag}. Complete proofs are in Appendices~\ref{app:directional_proof}--\ref{sec:stable}. Code will be released upon publication.

\noindent\textbf{Ethics.} ReGA is an inference-time method on publicly released text-to-image models and inherits their known risks of synthetic-media misuse, copyright, and biased outputs; we release no new model weights. The user study used unpaid volunteers recruited through an anonymous link, with no personal data collected.

\end{document}